\documentclass[twoside,11pt]{article}
\usepackage{jair, theapa, rawfonts}
\usepackage[T1]{fontenc}
\usepackage{lmodern}
\ShortHeadings{On Overfitting and Asymptotic Bias in RL with partial observability}
{Fran\c{c}ois-Lavet, Rabusseau, Pineau, Ernst \& Fonteneau}

\usepackage{url}            
\usepackage{booktabs}       
\usepackage{amsfonts}       
\usepackage{nicefrac}       
\usepackage{microtype}      

\usepackage{graphicx} 

\usepackage{algorithm}
\usepackage{algorithmic}

\usepackage{yfonts}
\usepackage{amsmath}
\usepackage{amsthm}
\usepackage{mathtools}
\usepackage{amsfonts}
\usepackage{amssymb}
\usepackage{amsthm}

\newtheorem{theorem}{Theorem}
\newtheorem{lemma}[theorem]{Lemma}
\newtheorem{proposition}[theorem]{Proposition}

\newtheorem{definition}{Definition}[section]

\theoremstyle{remark}
\newtheorem*{remark}{Remark}

\DeclareMathOperator*{\argmax}{argmax} 

\usepackage{xcolor,cancel}

\usepackage{tikz}
\usepackage{verbatim}
\usetikzlibrary{fit}
\usetikzlibrary{backgrounds}
\usetikzlibrary{positioning}
\usetikzlibrary{patterns}

\usepackage{textcomp}
\usepackage{mdframed}

\usetikzlibrary{shapes,arrows}
\usetikzlibrary{decorations.pathreplacing}
\usepackage{comment}
\usepackage{eurosym}

\usepackage{todonotes}
\newcommand{\guillaume}[1]{\todo[inline,color=blue!40,caption={}]{{\it Guillaume:~}#1}}

\begin{document}

\title{On Overfitting and Asymptotic Bias in Batch Reinforcement 
		Learning with Partial Observability}

\author{\name Vincent Fran\c{c}ois-Lavet \email vincent.francois-lavet@mail.mcgill.ca \\
       \name Guillaume Rabusseau \email guillaume.rabusseau@mail.mcgill.ca \\
       \name Joelle Pineau \email jpineau@cs.mcgill.ca \\
       \addr McGill University
       \AND
       \name Damien Ernst \email dernst@uliege.be \\
       \name Raphael Fonteneau \email raphael.fonteneau@uliege.be  \\
       \addr University of Liege}


\maketitle

\begin{abstract}
This paper provides an analysis of the tradeoff between asymptotic bias (suboptimality with unlimited data) and overfitting (additional suboptimality due to limited data) in the context of reinforcement learning with partial observability.
Our theoretical analysis formally characterizes that while potentially increasing the asymptotic bias, a smaller state representation decreases the risk of overfitting.
This analysis relies on expressing the quality of a state representation by bounding $L_1$ error terms of the associated belief states. 
Theoretical results are empirically illustrated when the state representation is a truncated history of observations, both on synthetic POMDPs and on a large-scale POMDP in the context of smartgrids, with real-world data.
Finally, similarly to known results in the fully observable setting, we also briefly discuss and empirically illustrate how using function approximators and adapting the discount factor may enhance the tradeoff between asymptotic bias and overfitting in the partially observable context.
\end{abstract}

\section{Introduction}
\label{intro}


This paper studies sequential decision-making problems that may be modeled as Markov Decision Processes (MDP) but for which the state is partially observable. This class of problems is called Partially Observable Markov Decision Processes (POMDPs) \cite{sondik1978optimal}.
Within this setting, we focus on decision-making strategies computed using Reinforcement Learning (RL). 
When the model of the environment is not available, RL approaches rely on observations gathered through interactions with the (PO)MDP, and, although some RL approaches have strong convergence guarantees, classic RL approaches are challenged by data scarcity. When acquisition of new observations is possible (the ``online'' case), data scarcity is gradually phased out using strategies balancing the exploration / exploitation (E/E) tradeoff. The scientific literature related to this topic is vast; in particular, Bayesian RL techniques  \cite{ross2011bayesian,ghavamzadeh2015bayesian} offer an elegant way of formalizing the E/E tradeoff. 


  
However, such E/E strategies are not applicable when the acquisition of new observations is not possible anymore. In the pure ``batch'' setting, the task is to learn the best possible policy from a fixed set of transition samples \cite{farahmand2011regularization,lange2012batch}.
Within this context, we propose to revisit RL as a learning paradigm that faces, similarly to supervised learning, a tradeoff between simultaneously minimizing two sources of error: an asymptotic bias and an overfitting error. 
The asymptotic bias (also simply called bias in the following) directly relates to the choice of the RL algorithm (and its parameterization). Any RL algorithm defines a policy class as well as a procedure to search within this class, and the bias may be defined as the performance gap between actual optimal policies and the best policies within the policy class considered. This bias does not depend on the set of observations. On the other hand, overfitting is an error term induced by the fact that only a limited amount of data is available to the algorithm.
This overfitting error vanishes as the size and the quality of the dataset increase.

In this paper, we focus on studying the interactions between these two sources of error, in a setting where the state is partially observable. Due to this particular setting, one needs to build a state representation from the sequence of observations, actions and rewards in a trajectory \cite{singh1994learning,aberdeen2003revised}.
By increasing the cardinality of the state representation, the algorithm may be provided with a more informative representation of the POMDP, but at the price of simultaneously increasing the size of the set of candidate policies, thus also increasing the risk of overfitting.
We analyze this tradeoff in the case where the RL algorithm provides an optimal solution to the frequentist-based MDP associated with the state representation (independently of the method used by the learning algorithm to converge towards that solution).
Our novel analysis relies on expressing the quality of a state representation by bounding $L_1$ error terms of the associated belief states, thus introducing the concept of $\epsilon$-sufficient statistics in the hidden state dynamics. 

Experimental results illustrate the theoretical findings on a distribution of synthetic POMDPs as well as a large-scale POMDP with real-world data. 
In addition, we illustrate the link between the variance observed when dealing with different datasets (directly linked to the size of the dataset) and overfitting, where the link is that variance leads to overfitting if we have a (too) large feature space.

By extending known results for MDPs, we also briefly discuss and illustrate how using function approximators and adapting the discount factor play a role in the tradeoff between bias and overfitting when the state is partially observable.
This has the advantage of providing the reader with an overview of key elements involved in the bias-overfitting tradeoff, specifically for the POMDP case.

The remainder of the paper is organized as follows. Section \ref{formalization} formalizes POMDPs, (limited) sets of observations and state representations. Section \ref{bias-over} details the main contribution of this paper: an analysis of the bias-overfitting tradeoff in learning POMDPs in the batch setting.
Section \ref{exp} empirically illustrates the main theoretical results, while Section \ref{conclusion} concludes with a discussion of the findings

\section{Formalization}
\label{formalization}

We consider a discrete-time POMDP \cite{sondik1978optimal} model $M$ described by the tuple $(\mathcal S,\mathcal A,T,R,\Omega,O,\gamma)$ where
\begin{itemize}
\item $\mathcal S$ is a finite set of states $\{1, \ldots, N_{\mathcal S}\}$,
\item $\mathcal A$ is a finite set of  actions $\{1, \ldots, N_{\mathcal A}\}$,
\item $T: \mathcal S \times \mathcal A \times \mathcal S \to [0,1]$ is the transition function (set of conditional transition probabilities between states),
\item $R: \mathcal S \times \mathcal A \times \mathcal S \to \mathcal R$ is the reward function, where $\mathcal R$ is a continuous set of possible rewards in a range $R_{max} \in \mathbb{R}^+$ (e.g., $[0,R_{max}]$ without loss of generality),
\item $\Omega$ is a finite set of observations $\{1, \ldots, N_{\Omega}\}$\label{ntn:N_Omega},
\item $O: \mathcal S \times \Omega \to [0,1]$\label{ntn:cond_obs} is a set of conditional observation probabilities, and
\item $\gamma \in [0, 1)$ is the discount factor.
\end{itemize}

The initial state is drawn from an initial distribution $b(s_0)$.
At each time step $t \in \mathbb N_0$, the environment is in a state $s_t \in \mathcal S$. At the same time, the agent receives an observation $\omega_t \in \Omega$ which depends on the state of the environment with probability $O(s_t, \omega_t)$ and the agent has to take an action $a_t \in \mathcal A$.
Then, the environment transitions to state $s_{t+1} \in \mathcal S$ with probability $T(s_t,a_t,s_{t+1})$ and the agent receives a reward  $r_t \in \mathcal R$ equal to $R(s_t, a_t, s_{t+1})$. 
In this paper, the conditional transition probabilities $T$, the reward function $R$ and the conditional observation probabilities $O$ are unknown. 
The only information available to the agent is the past experience it gathered while interacting with the POMDP. A POMDP is illustrated in Fig. \ref{fig:POMDP}.

\begin{figure}[ht!]
\centering
\def\dist{5}
\def\ra{0.9}
\begin{tikzpicture}[->,thick, scale=0.8]
\scriptsize
\tikzstyle{main}=[circle, minimum size = \ra, thick, draw =black!80, node distance = 12mm]
\tikzstyle{rr}=[rounded rectangle, rounded rectangle west arc=5pt, rounded rectangle east arc=50pt, minimum size = 7mm, thick, draw =black!80, node distance = 12mm]

\foreach \name in {0,...,2}
    \node[main, fill = white!100] (s\name) at (\dist*\name,0) {$s_\name$};
\foreach \name in {0,...,2}
    \node[main, fill = white!100] (omega\name) at (\dist*\name,-2) {$\omega_\name$};
\foreach \name in {0,...,2}
    \node[main, color=blue] (H\name) at (\dist*\name,-5) {$H_\name$};
\foreach \name in {0,...,2}
    \node[main, fill = white!100] (a\name) at (\dist*\name+1,-3) {$a_\name$};
\foreach \name in {0,...,1}
    \node[main, fill = white!100] (r\name) at (\dist*\name+4,-3) {$r_\name$};
\foreach \name in {0,...,1}{
    \node[] (tb\name) at (\dist*\name+2,0) {};
    \node[] (tf\name) at (\dist*\name+3,0) {};
}

\node[font=\Large] (dots) at (\dist*2+2,-2.5) {\ldots};

\draw [dashed,-] (-1,1) -- (11,1);
\draw [dashed,-] (-1,-1) -- (11,-1);

\foreach \name in {0,...,1}
    \draw [] plot [smooth, tension=2] coordinates { (a\name.north) (\dist*\name+\dist/2,0) (r\name.north) };

\foreach \cur/\next in {0/1,1/2}
       {
        \path (s\cur) edge (s\next);
        \path (s\cur) edge (omega\cur);
        \path [blue] (omega\cur) edge (H\cur);
        \path [blue] (H\cur) edge (H\next);
        \path [black!60!green] (H\cur) edge (a\cur);
        \node [black!60!green] (pol\cur) at (\dist*\cur+1.2,-4) {Policy};
        \path [blue] (r\cur) edge (H\next);
        \path [blue] (a\cur) edge (H\next);
       }
        \path (s2) edge (omega2);
        \path [blue] (omega2) edge (H2);
        \node [text width=2cm, align=center] (hid) at (11.5,0) {Hidden dynamics};

        \path [black!60!green] (H2) edge (a2);
        \node [black!60!green] (pol2) at (\dist*2+1.2,-4) {Policy};

\end{tikzpicture}
\caption{Graphical model of a POMDP.}
\label{fig:POMDP}
\end{figure}
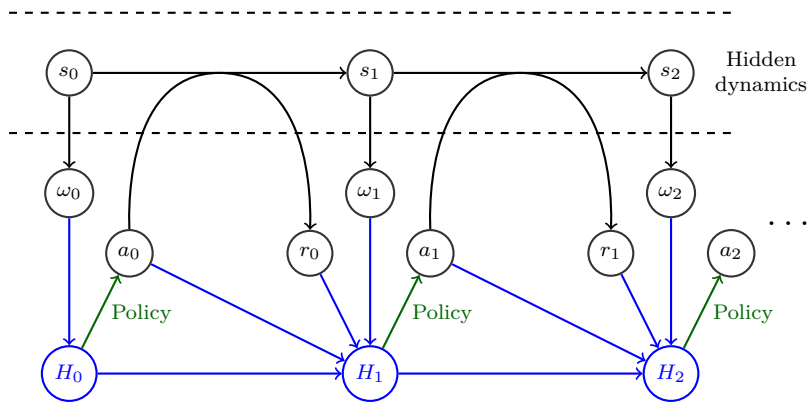

\subsection{Processing a History of Data}

Policies considered in this paper are mappings from (an ordered set of) observation(s) into actions.
A simple approach to build a space of candidate policies is to consider the set of mappings taking only the very last observation(s) as input \cite{whitehead1990active}. However, in a POMDP setting, this leads to candidate policies that are likely not rich enough to capture the system dynamics, thus suboptimal \cite{singh1994learning,wolfe2006pomdp}.
The alternative is to use a history of previously observed features to better estimate the hidden state dynamics \cite{mccallum1996reinforcement,littman2002predictive,singh2004predictive}. 
We denote by $\mathcal H_t=\Omega \times (\mathcal A \times \mathcal R \times \Omega)^{t}$ the set of histories observed up to time $t$ for $t \in \mathbb N_0$,  and by  $\mathcal H=\bigcup\limits_{t=0}^{\infty} \mathcal H_{t}$ the space of all possible observable histories.

A straightforward approach is to take the whole history $H_{t} \in \mathcal H$ as input of candidate policies. 
However, taking a too long history may have several drawbacks. 
Indeed, increasing the size of the set of candidate optimal policies generally implies: (i)~more computation to search within this set \cite{littman1994memoryless,mccallum1996reinforcement} and (ii)~an increased risk of including candidate policies suffering overfitting (see Section \ref{bias-over}).
In this paper, we are specifically interested in minimizing the latter overfitting drawback while keeping an informative state representation.


In this paper, we consider a mapping $\phi : \mathcal H \rightarrow \phi(\mathcal H)$, where $\phi(\mathcal H) = \{ \phi(H) | H \in \mathcal H   \} $ is of finite cardinality $|\phi(\mathcal H)|$.
On the one hand, we will show that when $\phi$ discards information from the whole history, the state representation $\phi(H)$ that the agent uses to take decision might depart from sufficient statistics, which can hurt performance.
On the other hand, we will show that it is beneficial to use a mapping $\phi$ that has a low cardinality $|\phi(\mathcal H)|$ to avoid overfitting. 
This can be intuitively understood since a mapping $\phi(\cdot)$ induces an upper bound on the number of candidate policies:
$|\Pi_{\phi(\mathcal H)}|~\le~|\mathcal A|^{|\phi(\mathcal H)|}$. 
In the following, we discuss this tradeoff formally.

Let us first introduce a notion of information on the latent hidden state $s$ through the notion of belief state \cite{cassandra1994acting}.

\begin{definition}
\label{belief_state}
The belief state $b(s | H)$ is defined as the vector of probabilities where the $i^{th}$ component ($i \in \{1, \ldots, N_{\mathcal S}\}$) is given by $\mathbb{P}(s=i \mid H)$, for any history $H \in \mathcal H$.
\end{definition}

\begin{definition}
The belief state $b_{\phi}\left(s | \phi(H)\right)$ is defined as the vector of probabilities where the $i^{th}$ component ($i \in \{1, \ldots, N_{\mathcal S}\}$) is given by $\mathbb{P}(s=i \mid \phi(H))$, for any history $H \in \mathcal H$
\footnote{Note that $s$ and $H$ are random variables and their exact distribution will depend on the context that is considered. 
For any given probability distribution $\mathcal D_H$ over histories: $H \sim \mathcal D_H$, the probability $P(s|\phi(H))$ is the expectation of the state when $\phi(H)$ is observed:
$b_\phi(s \mid \varphi)=\underset{H \sim \mathcal D_H, \varphi=\phi(H)}{\mathbb{E}} b(s \mid H)$.
}.
\end{definition}

Among all possible mappings $\phi$, we are particularly interested in the ones that extract enough information from the history to accurately 
capture the corresponding  belief state, with the notion of sufficient statistics \cite{kaelbling1998planning,aberdeen2007policy}. We thus define this notion in the context of the mapping $\phi$.
\begin{definition}
\label{sufficient_stat}
In a POMDP $M$, a statistic $\phi(H)$ is a sufficient statistic at the condition that $\forall s \in \mathcal S$:
$$\mathbb{P}(s \mid H) = \mathbb{P}(s \mid \phi(H)),$$
for $H \in \mathcal H$.
A mapping $\phi$ which provides sufficient statistics for all histories $H \in \mathcal H$ is called a sufficient mapping and is denoted as $\phi_0$.
\end{definition}

However, as mentioned previously, it may happen that a mapping $\phi$ does not capture enough information from a history to contain the same information than the belief state. 
One of the key notions on which our analysis relies is the one of ``approximately sufficient mappings'', i.e. mappings whose corresponding belief state lies in an $L_1$-ball of radius $\epsilon$ centered on $b(\cdot | H)$:  

\begin{definition}
\label{eps-sufficient_stat}
In a POMDP $M$, a statistic $\phi(H)$ is an $\epsilon$-sufficient statistic at the condition that it meets the following condition with $\epsilon \ge 0$ and with the $L_1$ norm:
$$\left\lVert b_{\phi}(\cdot | \phi(H)) - b(\cdot | H) \right\lVert_1  \le \epsilon,$$
for $H \in \mathcal H$.
A mapping $\phi$ that provides $\epsilon$-sufficient statistics for all histories $H \in \mathcal H$ is called an $\epsilon$-sufficient mapping and is denoted as $\phi_\epsilon(H)$.
\end{definition}

 

\subsection{Working with a Limited Dataset}

Let  $\mathcal M(\mathcal S,\mathcal A,\Omega, \gamma)$ be a set of POMDPs with fixed $\mathcal S$, $\mathcal A$, $\Omega$, and $\gamma$.
For any POMDP $M(T,R,O) \in \mathcal M$, we denote by $D_{M,\pi_s, N_{tr}, N_l}$ a random dataset generated according to a probability distribution $\mathcal D_{M,\pi_s, N_{tr}, N_l}$ over the set of $N_{tr}$ trajectories of length $N_l$. One such trajectory is defined as the observable history $H_{N_l} \in \mathcal H_{N_l}$ obtained in $M$ when starting from $s_0$ and following a stochastic sampling policy $\pi_s$ that ensures a non-zero probability of taking any action given an observable history $H \in \mathcal H$. For simplicity we denote $D_{M,\pi_s, N_{tr}, N_l}$, simply as $D \sim \mathcal D_M$.
For the purpose of the analysis, we also introduce the asymptotic dataset  $D_\infty = D_{M,\pi_s, N_{tr} \rightarrow \infty, N_{l} \rightarrow \infty} $ that would be theoretically obtained in the case where one could generate an infinite number of observations ($N_{tr} \rightarrow \infty$ and $N_{l} \rightarrow \infty$). 

In this paper, the algorithm cannot generate additional data. The challenge is to determine a high-performance policy (in the actual environment) while having only access to a fixed dataset $D$. We formalize this hereafter.

\subsection{Assessing the Performance of a Policy}

Let us consider stationary and deterministic control policies $\pi: \phi(\mathcal H) \rightarrow \mathcal A$ with $\pi \in \Pi$. Any particular choice of $\phi$ induces a particular definition of the policy space $\Pi$.
We introduce $V_{M}^\pi(\phi(H))$ with $H \in \mathcal H$ as the expected return obtained over an infinite time horizon when the system is controlled using policy $\pi$ 
in the POMDP $M$. For any given distribution $\mathcal D_H$ over histories, this is defined as:
\begin{align*}
V^\pi_M(\phi(H))=\underset{H' \sim \mathcal D_H:\ \atop \phi(H') = \phi(H)}{\mathbb{E}} V^\pi_M(H' \mid \phi),
\label{eq:expected_return}
\end{align*}
with $V^\pi_M(H \mid \phi)$ given by
\begin{align*}
V_{M}^\pi(H \mid \phi)=\mathbb{E} \left[\sum_{t=0}^{\infty} \gamma^{t} r_{t} | s_0 \sim b(\cdot|H), \pi \right],
\end{align*}
where we have $\mathbb P \big(\omega_t \mid s_t \big)=O(s_t,\omega_t)$, $a_t = \pi(\phi(H_t))$, $\mathbb P \big( s_{t+1} | s_{t}, a_t \big) = T(s_{t},a_t, s_{t+1})$ and $r_{t} = R \big(s_{t},a_t, s_{t+1} \big)$.

We also define $\pi^*$ as an optimal policy in $M$:
$$\pi^* \in \underset{\pi:\phi_0(\mathcal H) \rightarrow \mathcal A}{\argmax} \ 
 V_{M}^\pi(\phi_0(H_0)),$$
where $H_0$ 
is taken out of the distribution of initial observations (compatible with the distribution $b(s_0)$ of initial states through the conditional observation probabilities).

\section{Bias-overfitting in RL with Partial Observability}
\label{bias-over}


In this section, we study the performance difference (or gap) between the expected return that can be obtained following the policy built from limited data and the highest possible expected return that we would obtain if the algorithm had access to the POMDP parameters. In particular, we analyze how this performance gap can be decomposed into the sum of two terms: a term related to an asymptotic bias (suboptimality with unlimited data) and a term due to overfitting (additional suboptimality due to limited data).

\subsection{Importance of the Feature Space}
To study the importance of the feature space, let us assume that the policies built from limited data are optimal according to frequentist statistics, 
which allows removing from the analysis how the RL algorithm converges. 
In order to define the optimal policy according to frequentist statistics, let us first introduce a frequentist-based (augmented)  
MDP from the dataset $D$:
\begin{definition}
\label{augmented_DP}
With $M$ defined by $(S,A,T,R,\Omega,O,\gamma)$ and the dataset $D$ built from interactions with $M$, the frequentist-based augmented 
MDP $\hat {M}_{D, \phi}$, also denoted for simplicity $\hat {M}_{D}=(\mathrm{\Sigma},\mathrm{A},\hat{T},\hat{R},\mathrm{\Gamma})$, is defined with
\begin{itemize}
\item the state space: $\Sigma= \phi(\mathcal H)$\label{ntn:Sigma},
\item the action space: $\mathrm{A}=\mathcal A$\label{ntn:mathrmA},
\item the estimated transition function: for $\sigma,\sigma' \in \Sigma$ and $a \in \mathrm{A}$, $\hat{T}(\sigma,a,\sigma')$\label{ntn:hat_T} is the number of times we observe the transition $(\sigma, a) \rightarrow \sigma'$ divided by the number of times we observe $(\sigma,a)$ \footnote{if any $(\sigma,a)$ has never been encountered in a dataset, we arbitrarily set $\hat{T}(\sigma, a, \sigma') = 1/|\Sigma|, \forall \sigma'$. The theoretical results that follow are independent of how this case is treated.},
\item the estimated reward function: for $\sigma,\sigma' \in \Sigma$ and $a \in \mathrm{A}$, $\hat{R}(\sigma,a,\sigma')$\label{ntn:hat_R} is the mean of the rewards observed for the tuple $(\sigma, a, \sigma')$ \footnote{if any $(\sigma,a, \sigma')$ has never been encountered in a dataset, we arbitrarily set $\hat{R}(\sigma, a, \sigma')$ to the average of rewards observed over the whole dataset D. The theoretical results that follow are independent of how this case is treated.}, and
\item the discount factor $\Gamma \le \gamma$\label{ntn:Gamma}. 
\end{itemize}
\end{definition}

As long as the mapping $\phi$ is a sufficient mapping (thus denoted $\phi_0$), the asymptotic frequentist-based MDP (when unlimited data is available) actually gathers the relevant information from the actual POMDP. Indeed, when the POMDP is known (i.e. $T, R, O$ are known), the knowledge of $H_t$ allows one to obtain the belief state $b(s_t | H_t)$, calculated recursively 
thanks to the Bayes rule based on $b(s_t | H_t) = \mathbb P \left(s_{t}|\omega_t,a_t,b(s_{t-1} | H_{t-1})\right)$.
It is then possible to define, from the history $H \in \mathcal H$ and for any action $a \in \mathcal A$, the expected immediate reward as well as a transition function into the next observation $\omega'$:
\begin{equation*}
\begin{split}
R_{model-based} (H,a)= \sum_{s' \in \mathcal S} \sum_{s \in \mathcal S} b(s | H) T(s,a,s') R(s,a,s') \text{, and}\\
T_{model-based} (H,a,\omega')= \sum_{s' \in \mathcal S} \underbrace{ \sum_{s \in \mathcal S} b(s | H) T(s,a,s')}_\text{next belief state
} O\big( s', \omega' \big).
\label{ntn:T_model_based}
\end{split}
\end{equation*}
In the frequentist approach, this information is estimated directly from interactions with the POMDP in $\hat R$ and $\hat T$ without any explicit knowledge of the POMDP parameters.


We introduce 
$\mathcal V_{\hat {M}_{D}}^{\pi}(\sigma)$ with $\sigma \in \Sigma$
as the expected return obtained over an infinite time horizon when the system is controlled using a policy $\pi:\Sigma \rightarrow \mathrm A$ in the augmented decision process $\hat {M}_{D}$: 
$$\mathcal V_{\hat {M}_{D}}^\pi(\sigma)=\mathbb{E} \left[\sum_{t=0}^{\infty} \Gamma^{k} \hat r_{t} | \sigma_0=\sigma, \pi  \right],$$
where 
$\hat r_t$ is the reward s.t. $\hat r_t = \hat R(\sigma_t,a_t,\sigma_{t+1})$ and the transition is defined by $\mathbb P(\sigma_{t+1}|\sigma_t, a_t)=\hat T(\sigma_t, a_t, \sigma_{t+1})$. 

A policy $\pi$ is defined to be better than or equal to a policy $\pi'$ if its expected return is greater than or equal to that of $\pi'$ for all states. In an MDP, there is always at least one policy that is better than or equal to all other policies and this is an optimal policy \cite{sutton1998reinforcement}.
In the augmented
MDP $\hat {M}_{D}$, we denote the optimal policy as $\pi_{D, \phi}$ 
and we also call it the frequentist-based policy.
Let us now decompose the error of using a frequentist-based policy $\pi_{D, \phi}$ in the actual POMDP:

\begin{equation}
\begin{split}
 \underset{D\sim \mathcal D_M}{\mathbb E} \left[ V_M^{\pi^*}(\phi_0(H)) -V_M^{\pi_{D, \phi}}(\phi(H)) \right] =& \underbrace{\left(V_M^{\pi^*}(\phi_0(H))-V_M^{\pi_{D_\infty, \phi}}(\phi(H))\right)}_\text{\shortstack{bias function of dataset $D_\infty$ (function of $\pi_s$)\\
and frequentist-based policy $\pi_{D_\infty, \phi}$ (function of $\phi$ and $\Gamma$)}} \\
& + \underbrace{ \underset{D\sim \mathcal D_M}{\mathbb E} \left[ V_M^{\pi_{D_\infty, \phi}}(\phi(H))-V_M^{\pi_{D, \phi}}(\phi(H))\right]}_\text{\shortstack{overfitting due to finite dataset D (function of $\pi_s$, $N_l$, $N_{tr}$)\\ in the context of frequentist-based policy $\pi_{D, \phi}$ \\(function of $\phi$ and $\Gamma$)} }.
\end{split}
\label{bias-overfitting}
\end{equation}

The term \textit{bias} actually refers to an asymptotic bias when the size of the dataset tends to infinity, while the term \textit{overfitting} refers to the expected suboptimality due to the finite size of the dataset (and thus due to the variance in the estimated transition function and reward function). 

Selecting the feature space $\phi(\mathcal H)$ carefully allows building a class of policies that have the potential to accurately capture information from data (low bias), but also generalize well (low overfitting). 
On the one hand, using too many non-informative features will increase overfitting, as stated in Theorem \ref{union_bound} below.
On the other hand, a mapping $\phi(\mathcal H)$ that discards useful available information will suffer an asymptotic bias, as stated in Theorem \ref{theorem_bias} below (arbitrarily large depending on the POMDP and on the features discarded). 

We start by providing a bound on the bias, which is an original result based on the belief states via the $\epsilon$-sufficient statistic.

\begin{theorem}
\label{theorem_bias}
\textbf{"Bound on the bias":}
Let $M$ be a POMDP described by the 7-tuple $(\mathcal S,\mathcal A,T,R,\Omega,O,\gamma)$. 
Let $\hat M_{D_\infty}$ be an augmented MDP $(\Sigma,\mathrm{A},\hat{T},\hat{R},\Gamma=\gamma)$ estimated, according to Definition \ref{augmented_DP}, from a dataset $D_\infty$. 
Then, for any $\epsilon$-sufficient mapping $\phi=\phi_\epsilon$, the asymptotic bias can be bounded as follows:
\begin{equation}
\begin{aligned}
\underset{H \in \mathcal H}{\max} \bigl( V_M^{\pi^*}(\phi_0(H))-V_M^{\pi_{D_\infty, \phi}}(\phi(H))\bigl)
\le \frac{2 \epsilon R_{max}}{(1-\gamma)^3}.
\end{aligned}
\label{theorem}
\end{equation}
\end{theorem}

\begin{proof}
\label{app:theorem_bias}
We consider the frequentist-based MDP $\hat M_{D_\infty, \phi_0} (\Sigma_0, \mathrm A, \hat T, \hat R, \Gamma=\gamma)$, for $H \in \mathcal H$ and $a \in \mathcal A$, let us define
\begin{equation*}
\begin{split}
& \mathcal Q^{\pi_{D_\infty, \phi_0}}_{\hat M_{D_\infty, \phi_0}} (\phi_0(H),a)=\hat R'(\phi_0(H),a)+\gamma \sum_{\varphi \in \phi_0(\mathcal H)} \hat T(\phi_0(H),a,\varphi) \mathcal V^{\pi_{D_\infty, \phi_0}}_{\hat M_{D_\infty, \phi_0}}(\varphi),
\end{split}
\end{equation*}
where the reward 
$$\hat R'(\phi_0(H),a)=\underset{\varphi \in \phi_0(\mathcal H)}{\sum} \hat T(\phi_0(H),a,\varphi) \hat R(\phi_0(H),a, \varphi).$$
Then the main part of the proof is to demonstrate Proposition \ref{main_proposition} below.
From there, by applying Lemma 1 by \citeauthor{abel2016near}~\citeyear{abel2016near}, we have:
$$\left\lVert\mathcal V^{\pi_{D_\infty, \phi_0}}_{M_{D_\infty,\phi_0}}-\mathcal V^{\pi_{D_\infty, \phi_\epsilon}}_{M_{D_\infty,\phi_0}}\right\lVert_\infty \le \frac{2\frac{\epsilon R_{max}}{1-\gamma}}{(1-\gamma)^2}=\frac{2 \epsilon R_{max}}{(1-\gamma)^3}.$$
By further noticing that, when starting in $s_0$, $\hat M_{D_\infty,\phi_0}$ and $M$ provide an identical value function for a given policy $\pi_{D,\phi}$ and that $\pi_{D_\infty,\phi_0} \sim \pi^*$, i.e. $V_M^{\pi^*}=V_M^{\pi_{D_\infty,\phi_0}}$, the theorem follows.
\end{proof}

\begin{remark}
As compared to \citeauthor{hutter2014extreme}~\citeyear{hutter2014extreme} and \citeauthor{abel2016near}~\citeyear{abel2016near}, this bound relates directly to the capacity of the mapping $\phi(H)$ to retrieve sufficient information on the latent hidden state.
As compared to PBVI \cite{pineau2003point} and similar approaches, we do not make the assumption that $T$, $R$ and $O$ are known and, as such, they need to be estimated from data.
\end{remark}

We now provide Proposition \ref{main_proposition}, which is the key result required in the proof of Theorem~\ref{theorem_bias}.
\begin{proposition}
\label{main_proposition}
Let $\phi_\epsilon$ be an $\epsilon$-sufficient mapping, and let $\phi_0$ be a sufficient mapping.
Then, for any $H^{(1)}, H^{(2)}$ such that $\phi_\epsilon(H^{(1)})=\phi_\epsilon(H^{(2)})$, we have
\begin{align*}
\underset{a}{\max} & \left|\mathcal Q^{\pi_{D_\infty, \phi_0}}_{\hat M_{D_\infty, \phi_0}}(\phi_0(H^{(1)}),a) -\mathcal Q^{\pi_{D_\infty, \phi_0}}_{\hat M_{D_\infty, \phi_0}}(\phi_0(H^{(2)}),a)\right| \quad \le \epsilon \frac{R_{max}}{(1-\gamma)}. \\
\end{align*}
\end{proposition}
\begin{proof}

For this proposition, we rely on the fact that since $\phi_\epsilon(H^{(1)})=\phi_\epsilon(H^{(2)})$, we are able to bound the L1 error terms of the associated belief states of $H^{(1)}$ and $H^{(2)}$. This is illustrated in Figure~\ref{fig:mapping}. From that bound, we then present two different ways of independent interest to prove Proposition \ref{main_proposition}.
The details of the proofs are given in Appendix \ref{main_proof}.

\begin{figure}[ht!]
\centering
\resizebox{0.4\textwidth}{!}{%
\begin{tikzpicture}[
dot/.style    = {anchor=base,fill,circle,minimum size=0.1cm, inner sep=0.2pt}]

\draw plot [smooth cycle] coordinates {(0.5,1.1)(1.3,1.2)(2.8,1.5)(2.9,2.5)(2.8,3.8)(1.2,3.5)(0.2,1.5)} node at (1.5,0.6) {History space $\mathcal H$};
\draw plot [smooth cycle] coordinates {(5.5,3.25) (7.3,3.35) (8.3, 3.2) (8.4,3.5) (8.5,4.65) (8.8,5.75) (5.3,5.75) (5.,4.45) (5.3,3.2) } node at (7.,6.5)[text width=4cm] {\centering Space of the \\$\epsilon$-sufficient statistics};
\draw plot [smooth cycle] coordinates {(5,-0.75) (6,-0.65) (6.8, -0.8) (8.5,-0.5) (8.5,0.65) (6.8,1.75) (5.2,1.75) (5.3,0.45) (4.8,-0.15) } node at (6.5,-1.4) {Belief space};

\path[->] (3.1,3.2) edge [bend left] node[label={[xshift=0.2cm, yshift=0.1cm]$\phi_\epsilon(\mathcal H)$}] {} (4.8,4.5);
\path[->] (3.1,2.2) edge [bend right] node[label={[xshift=0.5cm, yshift=0.cm]$b(\cdot \mid \mathcal H)$}] {} (5.0,1.2);

\node [dot, fill=blue, label={[xshift=0cm, yshift=0.cm, blue]$H^{(1)}$}] (dot1_belief) at (1.8,3) {};
\node [dot, fill=black!45!green, label={[xshift=0cm, yshift=-0.8cm, black!45!green]$H^{(2)}$}] (dot2_belief) at (1.8,2) {};

\node [dot, fill=blue, label={[xshift=0.2cm, yshift=0.cm, blue]$b(s \mid H^{(1)})$}] (dot1_belief) at (6,1) {};
\node [dot, fill=black!45!green, label={[xshift=0.2cm, yshift=-0.8cm, black!45!green]$b(s \mid H^{(2)})$}] (dot2_belief) at (6,0) {};
\path[<->] (dot1_belief) edge (dot2_belief);
\node at (7.2,0.5) {$\left\lVert \cdot \right\lVert_1 \le 2 \epsilon$};

\node [dot,label={[xshift=0.8cm, yshift=0.1cm]$\textcolor{blue}{\phi_\epsilon(H^{(1)})}=\textcolor{black!45!green}{\phi_\epsilon(H^{(2)})}$}] (dot1_phie) at (6,4.5) {};

\end{tikzpicture}
}
\caption{Illustration of the $\phi_\epsilon$ mapping and the belief for $H^{(1)}, H^{(2)} \in \mathcal H$: $\phi_\epsilon(H^{(1)})=\phi_\epsilon(H^{(2)})$.}
\label{fig:mapping}
\end{figure}
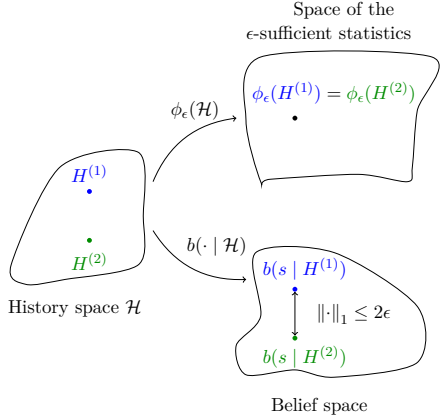

The first proof makes use of a tree of possible future observations, rewards and corresponding actions given a policy $\pi$, and we show that when starting from $H^{(1)}, H^{(2)}$ such that $\phi_\epsilon(H^{(1)})=\phi_\epsilon(H^{(2)})$, the bound holds.

We also provide an alternative proof that makes use of the formalism of the bisimulation metric \cite{ferns2004metrics} along with the data processing inequality.
Intuitively, the proof relies on the fact that the histories $H^{(1)}$ and $H^{(2)}$ are close according to a bisimulation metric that measures the behavioral similarity (future rewards and observations). The data processing inequality is used for finding that distance as a function of  the L1 error terms of the associated belief states.
\end{proof}


We now provide a bound on the overfitting error that monotonically grows with $|\phi(\mathcal H)|$. Theorem \ref{union_bound} shows that using a large set of features 
allows a larger policy class, hence potentially leading to a stronger drop in performance when the available dataset $D$ is limited (the bound decreases proportionally to $\frac{1}{\sqrt n}$). 
A theoretical analysis in the context of  MDPs with a finite dataset was performed by \citeauthor{jiang2015abstraction}~\citeyear{jiang2015abstraction}.

\begin{theorem}
\label{union_bound}
\textbf{"Bound on the overfitting":}
Let $M$ be a POMDP described by the 7-tuple $(\mathcal S,\mathcal A,T,R,\Omega,O,\gamma)$.
Let $\hat M_{D}$ be an augmented MDP $(\Sigma,\mathrm{A},\hat{T},\hat{R},\Gamma=\gamma)$ estimated, according to Definition \ref{augmented_DP}, from a dataset $D$ with the assumption that 
$D$ has $n$ transitions from any possible pair $(\phi(H),a) \in \phi(\mathcal H) \times \mathrm A$ (sampled i.i.d according to $\mathcal D_M$). 
Then the overfitting due to using the frequentist-based policy $\pi_{D, \phi}$ instead of $\pi_{D_\infty, \phi}$ in the POMDP $M$ can be bounded as follows:
\begin{equation}
\begin{aligned}
\underset{H \in \mathcal H}{\max} & \bigl( V_M^{\pi_{D_\infty, \phi}}(\phi(H))-V_M^{\pi_{D, \phi}}(\phi(H)) \bigl)
\le \frac{2 R_{max}}{(1-\gamma)^2} \sqrt{ \frac{1}{2n}  ln\left(\frac{2 |\phi(\mathcal H)| |\mathrm A|^{1+|\phi(\mathcal H)|}}{\delta}\right) },
\end{aligned}
\end{equation}
with probability at least $1-\delta$.
\end{theorem}

\begin{proof}
The proof of Theorem \ref{union_bound} is deferred to Appendix \ref{app:union_bound}.
\end{proof}


Overall, Theorems \ref{theorem_bias} and \ref{union_bound} can help choose a good state representation for POMDPs as they provide bounds on the two terms that appear in the bias-overfitting decomposition of Equation \ref{bias-overfitting}. For example, an additional feature in the mapping $\phi$ has an overall positive effect only if it provides a significant increase of information on the belief state (i.e. if it allows one to obtain a more accurate knowledge of the underlying state of the MDP defined by $T$ and $R$ when given $\phi(H)$). This increase of information must be significant enough to compensate for the additional risk of overfitting when choosing a large cardinality of $\phi(\mathcal H)$.
Note that one could combine the two bounds to theoretically define an optimal choice of the state representation with lower bound guarantees regarding the bias-overfitting tradeoff. In practice, as the two bounds are loose, other techniques described in Section \ref{selec_params} are usually more useful.

\paragraph{Related work}
Partial observability is very common in real world domains and though there has been many works in state abstraction and homomorphisms in the MDP setting \cite{ravindran2004approximate,ferns2004metrics}, there has been relatively little in the POMDP setting. 
One related work in POMDPs by \citeauthor{castro2009equivalence}~\citeyear{castro2009equivalence} has discussed the parallel between the notion of bisimulation and a notion of trace equivalence, under which states are considered equivalent if they generate the same conditional probability distributions over observation sequences (where the conditioning is on action sequences).

In this paper, we introduce the definition of $\epsilon$-sufficient statistics in POMDPs and derive bounds for the bias based on this property. 
An insightful part of the proof of Proposition \ref{main_proposition} is to use the bisimulation metric \cite{ferns2004metrics} and the data processing inequality.
Note that the bisimulation metrics may also be used to take into account how the errors on the belief states may have less of an impact at the condition that the hidden states affected by these errors are close according to the bisimulation metric. 
In case there is some knowledge on the underlying dynamics, one could also extend the notion of the bisimulation metric to allow certain distinct actions to be essentially equivalent \cite{arun2006bisimilarities}.
In that context, \citeauthor{taylor2009bounding}~\citeyear{taylor2009bounding} have generalized the notion of the bisimulation metric to a lax bisimulation metric, which may take into account some symmetries and special structures.

\subsection{Discussion on Function Approximators}
As described earlier, a straightforward mapping $\phi(\cdot)$ may be obtained by discarding some features from the observable history. 
In addition, the theoretical work developed in this paper can be useful to understand how a specific design of deep neural networks may work well (or not). Indeed, deep neural networks can be seen as a composition of many learnable mappings (constrained by the design of the neural network) and we provide bounds that depend on the property of a given mapping $\phi$ of the inputs (e.g., the first layers of a deep Q-network).
If there is, for instance, an attention mechanism in the first layers of a deep neural network, the mapping made up of those first layers can be analyzed through the bounds developed in our work (e.g., it should not discard important observations that prevent the mapping to be close to some sufficient statistics).
As another example, our work helps explain why basic recurrence in a neural network may not be well-suited for POMDPs that require a long history of observations for approaching some sufficient statistics. Indeed, basic recurrent cells are known to have difficulties to convey the long-term dependencies in sequences (i.e. features from the first time steps in a long time series) and LSTMs \cite{hochreiter1997long} or other variants should be preferred.

Note that we could add a theorem showing that using, for instance, Rademacher complexity would allow providing other bounds (potentially tighter) than Theorem \ref{union_bound}. However, such a bound is usually of little interest in practice, specifically concerning deep learning, because the bounds based on complexity measures fail to provide insights on the generalization capabilities of neural networks \cite{zhang2016understanding}. Indeed, it has been empirically demonstrated that deep neural networks have strong generalization capabilities even with a high number of parameters (hence a potentially large complexity). In other words, a very loose bound on the overfitting error due to a large number of parameters may not lead to a performance drop in practice.

\subsection{Selection of the Parameters with Validation or Cross-Validation to Balance the Bias-Overfitting Tradeoff}
\label{selec_params}


In the batch setting case, the selection of the policy parameters to effectively balance the bias-overfitting tradeoff can be done similarly to that in supervised learning (e.g., cross-validation) as long as the performance criterion can be estimated from a subset of the trajectories from the dataset $D$ not used during training (validation set).
One possibility is to fit an MDP model from data via the frequentist approach (or regression), and evaluate the policy against the model (with Monte-Carlo simulations). Another approach is to use the idea of importance sampling \cite{precup2000eligibility}. A mix of the two approaches has also been developed \cite{jiang2016doubly,thomas2016data}. 

\paragraph{Importance of the discount factor $\Gamma$ used in the training phase:}
Besides the elements related to feature selection and function approximators, artificially lowering the discount factor can also be used to improve the performance of the policy when solving MDPs with limited data \cite{petrik2009biasing,jiang2015dependence}.
In the partially observable setting, these results can be transferred to the frequentist-based MDP $(\Sigma,\mathrm A,\hat T,\hat R,\Gamma)$ by choosing $\Gamma<\gamma$, which introduces a bias but reduces the risk of overfitting. 



\section{Experiments}
\label{exp}

This section provides empirical illustrations of the theoretical results both on synthetic POMDPs and on a large-scale POMDP in the context of smartgrids, with real-world data. 
On the synthetic POMDPs, we first illustrate the main theoretical results of this paper related to the state representation of POMDPs and we also illustrate the use of function approximators and the training discount factor $\Gamma$.
On the large-scale POMDP, we illustrate that an efficient feature selection process can be useful, even when used in addition to function approximators and a biased discount factor.

\subsection{Synthetic POMDPs}
\subsubsection{Protocol}

In order to be representative of a diversity of environments, we provide results on a distribution of POMDPs. 
We randomly sample $N_P$ POMDPs such that $N_{\mathcal S} = 5$, $N_{\mathcal A} = 2$ and  $N_{\Omega} = 5$ (except when stated otherwise) from a distribution $\mathcal P$ that we refer to as Random POMDP. The distribution $\mathcal P$ is fully determined by specifying a distribution over the set of possible transition functions $T (\cdot, \cdot, \cdot)$, a distribution over the set of reward functions $R(\cdot,\cdot,\cdot)$, and a distribution over the set of possible conditional observation probabilities $O(\cdot, \cdot)$.

Random transition functions $T (\cdot, \cdot, \cdot)$ are drawn by assigning, for each entry $(s,a,s')$, a zero value with probability 3/4, and, with probability 1/4, a non-zero entry with a value drawn uniformly in $[0, 1]$. For all $(s,a)$, if all $T(s,a,s')$ are zeros, we enforce one non-zero value for a random $s' \in \mathcal S$. Values are normalized.
Random reward functions are generated by associating to all possible $(s,a,s')$ a reward sampled uniformly and independently from $[-1, 1]$. 
Random conditional observation probabilities $O(\cdot, \cdot)$ are generated in the following way: the probability of observing $o^{(i)}$ when in state $s^{(i)}$ is equal to 0.5, while all other values are chosen uniformly randomly so that it is normalized for any $s$.
For all POMDPs, we have $\gamma = 1$ and $\Gamma = 0.95$ if not stated otherwise and we truncate the trajectories to a length of $N_l=100$ time steps.


For each generated POMDP $P  \sim \mathcal P$, we generate 20 datasets $D \in \mathcal D_P$ where $\mathcal D_P$ is a probability distribution over all possible sets of $n$ trajectories ($n \in [2,5000]$); where each trajectory is made up of a history $H_{100}$ of 100 time steps, when starting from an initial state $s_0 \in \mathcal S$ and taking uniform random decisions.
Each dataset $D$ induces a policy $\pi_{D, \phi}$, and we want to evaluate the expected return of this policy while discarding the variance related to the stochasticity of the transitions, observations and rewards. 
To do so, policies are tested with 1000 rollouts of the policy.
For each POMDP $P$, we are then able to get an estimate of the average score $\mu_P$ which is defined as:
$$
\mu_P=\underset{D \sim \mathcal D_P}{\mathbb E} \  \underset{}{\mathbb E} \left[ \sum_{t=0}^{N_l} \gamma^t r_t | s_0, \pi_{D, \phi} \right]
.$$
We are also able to get an estimate of a parametric variance $\sigma_P^2$ defined as:
$$
\sigma_P^2=\underset{D \sim \mathcal D_P}{\text{Var}} \  \underset{}{\mathbb E} \left[ \sum_{t=0}^{N_l} \gamma^t r_t | s_0, \pi_{D, \phi} \right]
.$$

\subsubsection{History Processing}
\label{choose_history}

We show experimentally that any additional feature from the history $H_t$ is likely to reduce the asymptotic bias, but may also increase the overfitting.
For any history length $h$, we consider the mapping $\phi_h$ that extracts the current observation and the last $h-1$ (observation, action) tuples. In the experiments, we compare the policies $\pi_{D,\phi_h}$ for $h=1,2,3$.

\begin{figure}[ht!]
    \centering
        \includegraphics[width=230pt]{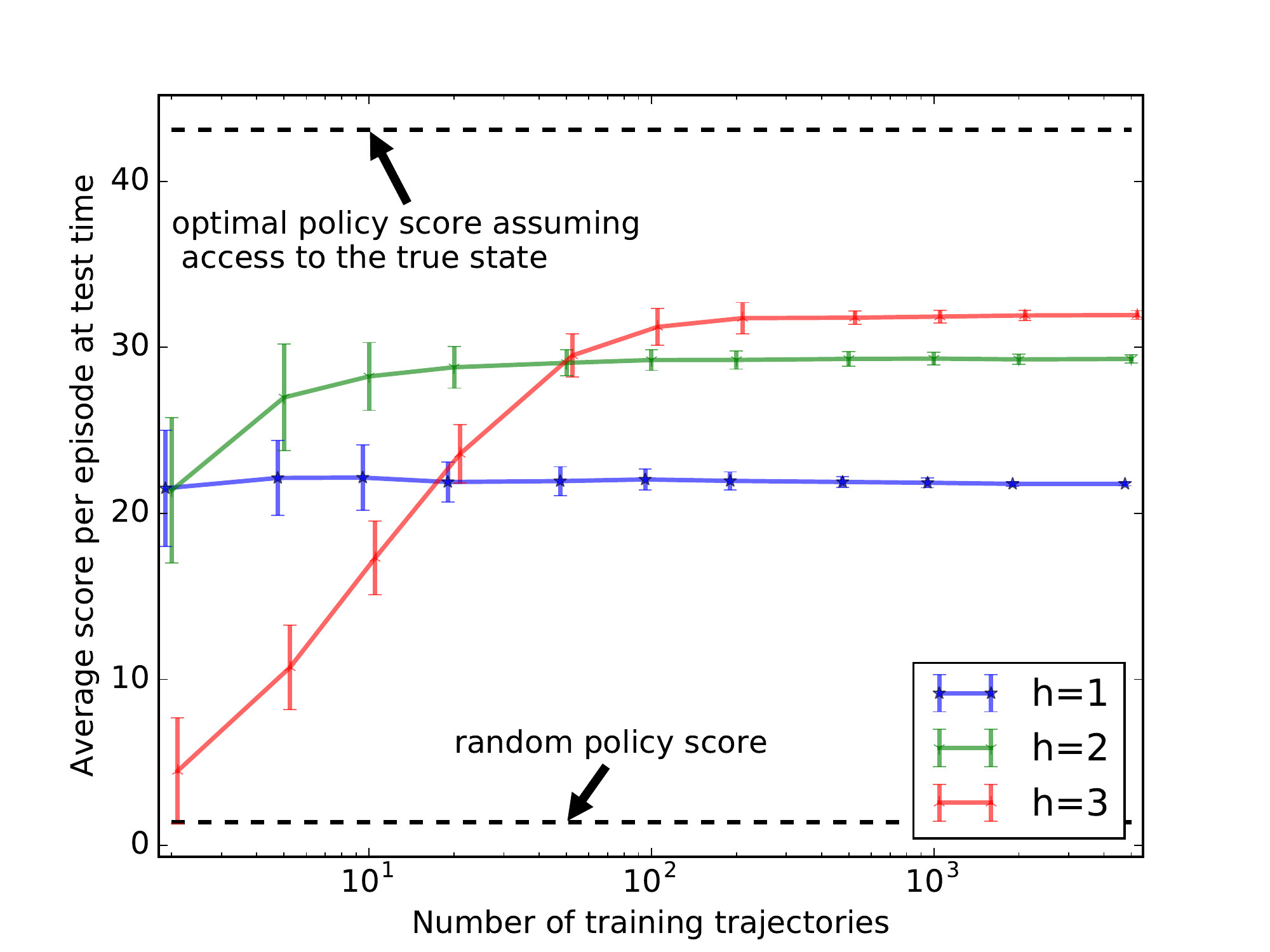}
     \caption{Evolution (as a function of the size of the dataset) of estimated values of $\underset{P \sim \mathcal P}{\mathbb E} \mu_P  \pm  \underset{P \sim \mathcal P}{\mathbb E} \sigma_P$ computed from a sample of $N_P=50$ POMDPs drawn from $\mathcal P$. The bars are used to represent the variance observed when dealing with different datasets drawn from a distribution; note that this is not a usual error bar.}
    \label{fig:Random_POMDP_results}
\end{figure}


The values $\underset{P \sim \mathcal P}{\mathbb E} \mu_P$ and $\underset{P \sim \mathcal P}{\mathbb E} \sigma_P$
 are displayed in Figure~\ref{fig:Random_POMDP_results}.
One can observe that a small set of features (small history) appears to be a better choice (in terms of total bias) when the dataset is small (only a few trajectories). With an increasing number of trajectories, the optimal choice in terms of number of features ($h=1, 2$ or $3$) also increases.
In addition, one can also observe that the expected variance of the score decreases as the number of samples increases. As the variance decreases, the risk of overfitting also decreases, and it becomes possible to target a larger policy class (using a larger feature set).

The overfitting error is also linked to the variance of the value function estimates in the frequentist-based MDP.
When these estimates have a large variance,  an overfitting term appears because of a higher chance of picking one of the suboptimal policies, as illustrated in Figure~\ref{fig:Random_POMDP_results_distrib}. 


\begin{figure}[ht!]
    \centering
        \includegraphics[width=260pt]{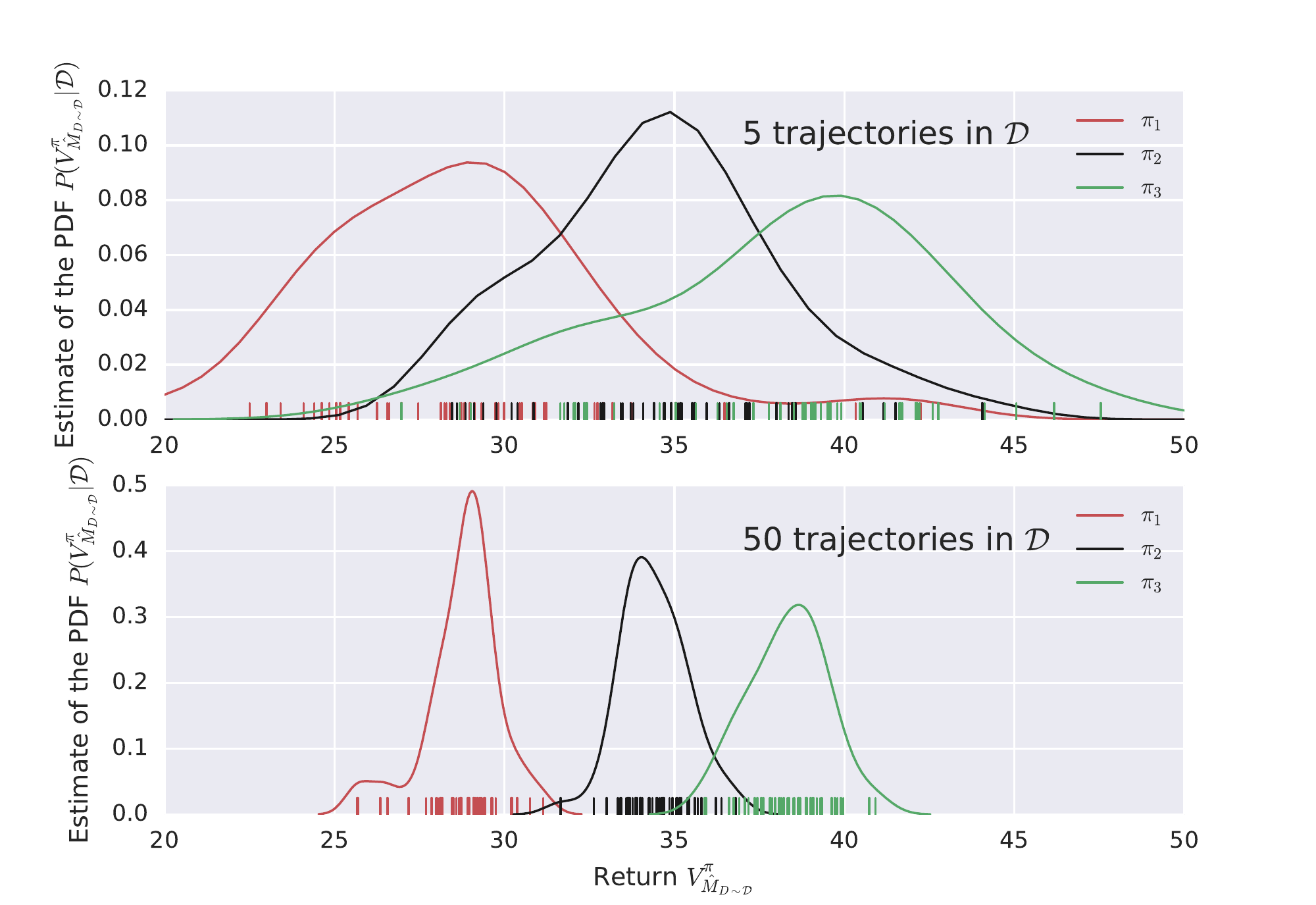}
     \caption{Illustrations of the return estimates in the augmented MDP $\hat {M}_{D}$ ($h=2$) for three different policies. 
     The policies are selected specifically for illustration purposes based on the criterion $V^{\pi_D}_M$; the best performing (in green), the worst performing (in red) and the median performing were selected in a set of 50 policies built when $D \sim \mathcal D$ has 5 trajectories of data from the actual POMDP. 
     On the actual POMDP, the expected returns are $V^{\pi_1}_M=28$, $V^{\pi_2}_M=33$, $V^{\pi_3}_M=42$ (in general, these values need not be the same as the expected value of the probability distribution in the two graphs).}
    \label{fig:Random_POMDP_results_distrib}
\end{figure}

\subsubsection{Function Approximator and Discount Factor}

We also illustrate the effect of using function approximators on the bias-overfitting tradeoff. To do so, we process the output of the state representation $\phi(\cdot) $ into a deep Q-learning scheme (technical details are given in Appendix \ref{app:Q_learn}).
We can see in Figure~\ref{fig:Random_POMDP_NN} that deep Q-learning policies suffer less overfitting as compared to the frequentist-based approach (lower degradation of performance in the low-data regime) even though using a large set of features still leads to more overfitting than a small set of features. We also see that deep Q-learning policies do not introduce an important asymptotic bias (identical performance when a lot of data is available) because the neural network architecture is rich enough.
Note that the variance is slightly larger than in Figure~\ref{fig:Random_POMDP_results}, and does not vanish to 0 with additional data. This is due to the additional stochasticity induced when building the Q-value function with neural networks (note that when performing the same experiments while taking the average recommendation of several Q-value functions, this variance decreases with the number of Q-value functions).

\begin{figure}
    \centering
        \includegraphics[width=230pt]{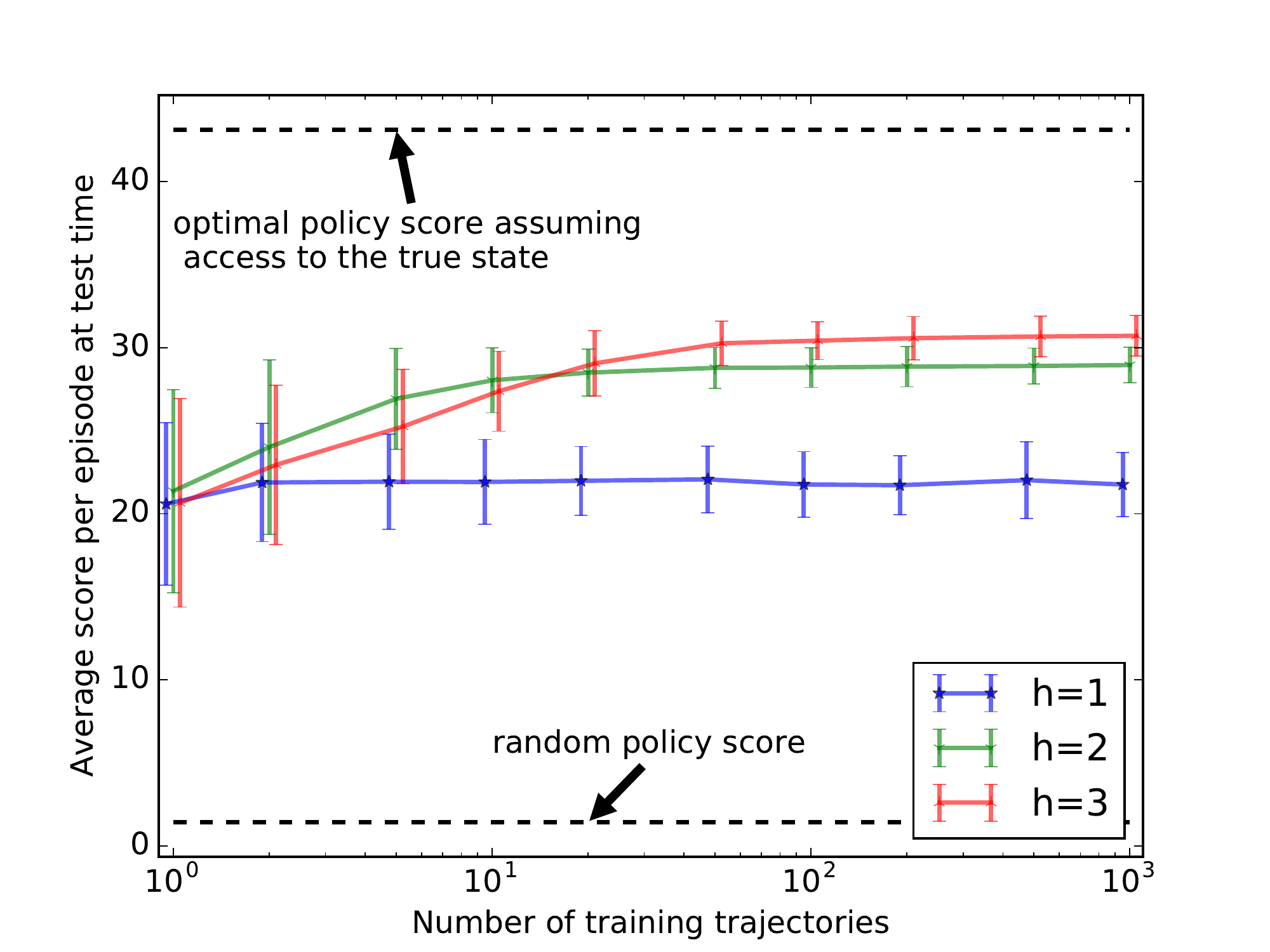}
     \caption{Evolution (as a function of the size of the dataset) of estimated values of $\underset{P \sim \mathcal P}{\mathbb E} \mu_P  \pm  \underset{P \sim \mathcal P}{\mathbb E} \sigma_P$ computed from a sample of $N_P=50$ 
      POMDPs drawn from $\mathcal P$ (same as Figure~\ref{fig:Random_POMDP_results}) 
     with neural network as a function approximator. 
The bars are used to represent the variance observed when dealing with different datasets drawn from a distribution; this is not a usual error bar.}
    \label{fig:Random_POMDP_NN}
\end{figure}

Finally, we empirically illustrate in Figure~\ref{fig:Random_POMDP_discount} the effect of modifying the discount factor $\Gamma$.
When the training discount factor is lower than the one used in the actual POMDP ($\Gamma<\gamma$), there is an additional bias term, while when a high discount factor is used ($\Gamma$ close to $1$) with a limited amount of data, overfitting increases. In our experiments, the influence of the discount factor is more subtle as compared to the impact of the state representation and the function approximator. The influence is nonetheless clear: it is better to have a low discount factor when only a few data points are available, and it is better to have a high discount factor when a lot of data is available, which is in line with previous analyses for MDPs \cite{jiang2015dependence}.

\begin{figure}[!ht]
    \centering
        \includegraphics[width=230pt]{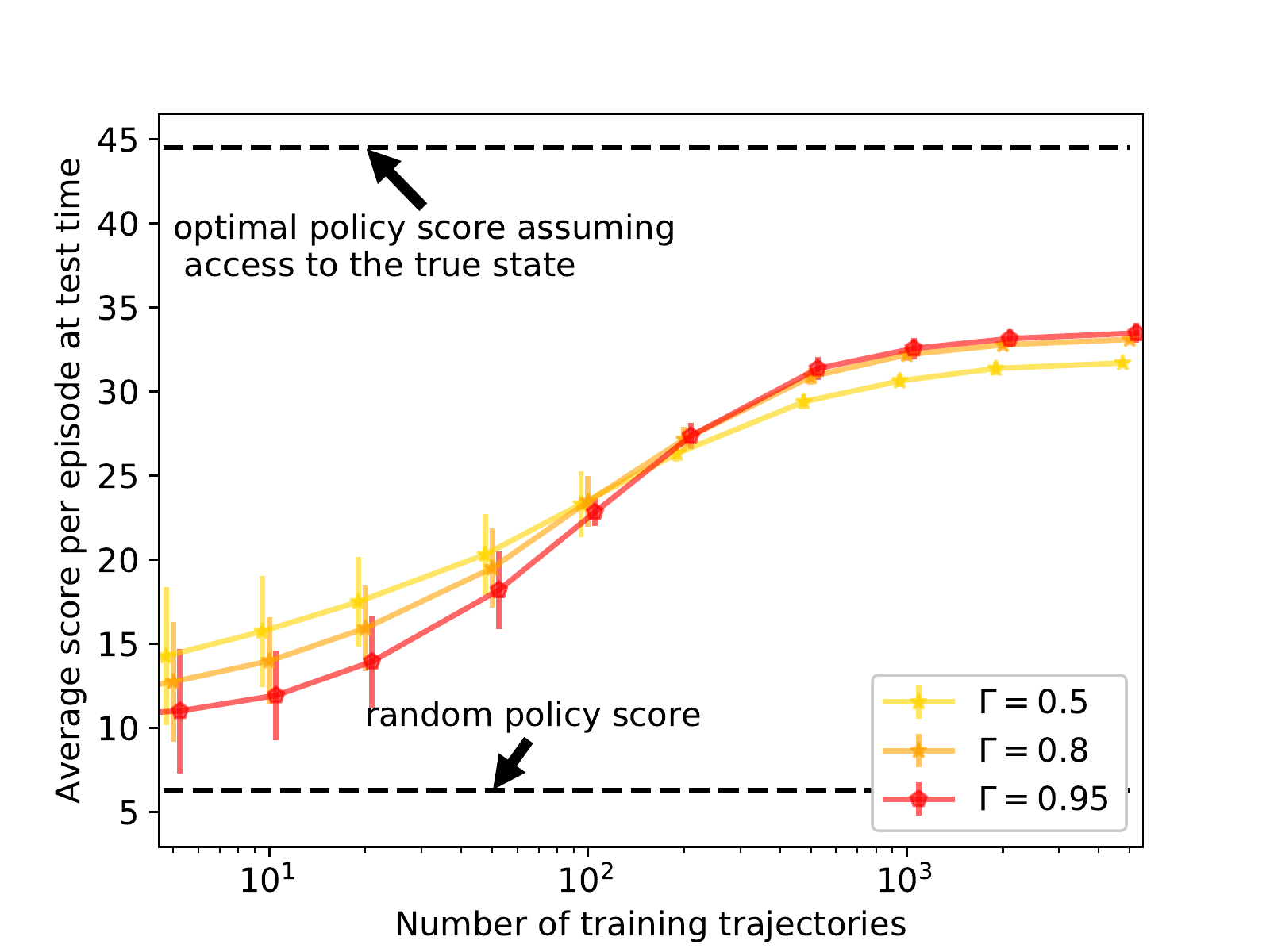}
        \caption{Evolution in the frequentist-based case (as a function of the size of the dataset) of estimated values of $\underset{P \sim \mathcal P}{\mathbb E} \mu_P  \pm  \underset{P \sim \mathcal P}{\mathbb E} \sigma_P$ computed from a sample of $N_P=10$ 
      POMDPs drawn from $\mathcal P$ with $N_{\mathcal S} = 8$ and  $N_{\Omega} = 8$ ($h=3$).
      The bars are used to represent the variance observed when dealing with different datasets drawn from a distribution; this is not a usual error bar.}
        \label{fig:Random_POMDP_discount}
\end{figure}

\subsection{Real-world Application in Smartgrids}

The microgrid model considered here was introduced by \citeauthor{franccois2016deep}~\citeyear{franccois2016deep} and we next describe the main elements \footnote{The source code is available at \url{https://github.com/VinF/deer/}.}. 

\subsubsection{Microgrid Benchmark} 
The microgrid is powered by photovoltaic (PV) panels combined with both long-term storage (hydrogen-based) and short-term storage (such as, for instance, LiFePO$_4$ batteries).
These two types of storage aim at fulfilling, at best, the demand by addressing the seasonal and daily fluctuations of solar irradiance. 
Distinguishing short- from long-term storage is mainly a question of cost: batteries are too expensive to be used for addressing seasonal variations.




Operating the microgrid is formalized as a POMDP over discrete time steps of one hour. The observations at each time step are made up of the consumption $ \left( c_t \right)$, the production $\left( \psi_t \right)$ and the level of energy in the battery ($s^{B}_t$).
The mapping $\phi(\cdot)$ is defined such that
$$\phi(H_t) = \left[ [c_{t-h^{c}}, \ldots, c_{t-1}], [\psi_{t-h^{p}}, \ldots, \psi_{t-1}], s^{B}_t\right],$$
where $h^c=h^p$ are the lengths of the time series considered for the consumption and production, respectively.

The instantaneous reward signal $r_t$ is obtained by adding the revenues generated by the hydrogen production $r^{H_2}$ and the penalties $r^-$\label{ntn:r_h2_r-} due to the value of loss load: 
\begin{equation}
r_t=r(a_t, d_t)=r^{H_2}(a_t) + r^-(a_t, d_t),
\end{equation}
where $d_t$ denotes the net electricity demand, which is the difference between the local consumption $c_t$ and the local production of electricity $\psi_t$.
The penalty $r^-$ is proportional to the total amount of energy that was not supplied to meet the demand, with the cost incurred per kilowatt-hour ($kWh$) not supplied within the microgrid set to $2$ \euro$/kWh$ (corresponding to a value of loss load).
The revenues (or cost) generated by the hydrogen production $r^{H_2}$ is proportional to the total amount of energy transformed (or consumed) in the form of hydrogen, where the revenue (or cost) per $kWh$ of hydrogen produced (or used) is set to $0.1$ \euro$/kWh$.

A synthetic residential consumption profile is considered with a daily consumption of $18kWh$. The PV production profile comes from actual data of a residential customer located in Belgium (average production changes by a factor of about 1:5 between summer and winter). A fixed sizing of the microgrid is considered. The size of the battery is $x^B=15kWh$, the instantaneous power of the hydrogen storage is $x^{H_2}=1.1kW$ and the peak power generation of the PV installation is $x^{PV}=12kW_p$ ($kW_p$ stands for kilowatt-peak, which is the maximum electric power that can be supplied by the PV panels).

The action $a_t$ is an element of the action space made up of three discrete actions: $\mathcal A=\{a^{(0)},a^{(1)},a^{(2)}\}$. The possible actions relate to whether the microgrid creates hydrogen from electricity, creates electricity from hydrogen or leaves the long-term storage unused. The battery adapts to store excess electrical power (except when the battery is full) while avoiding any loss load (except when the battery is empty, which leads to negative rewards). 

Note that in this specific application, using a history of observations from the previous time steps provides information on the latent features (of the state of the POMDP) such as the time of the day, the weather, the season (PV production and consumption time series are conditionally dependent on these latent features).

Also note that the dynamics of the system depend on exogenous time series (production and consumption time series) for which the agent only has finite data.
In order to evaluate the capacity of the agent to generalize, we can break down the exogenous time series into a part that will be used for training and one part that will be used for validation \cite{franccois2016deep}. The final performance is then evaluated on the environment with unseen time series.

\subsubsection{Splitting the Times Series to Avoid Overfitting}
\label{app:splitting}

In this microgrid domain, the agent is provided with up to two years of actual past realizations of the consumption $ \left( c_t \right)$ and the production $\left( \psi_t \right)$. 
These past realizations are split into a training environment (one year) and a validation environment (one year). The training environment is used to train the policy, while the validation environment is used at each epoch \footnote{An epoch is defined as the set of all iterations required to go through the whole year of the exogenous time series $c_t$ and $\psi_t$. Each iteration is made up of a transition of one time-step in the environment as well as a gradient step of all parameters $\theta$ of the Q-network.} 
to estimate how well the policy performs on the undiscounted sum of rewards, and it selects the best (approximated) Q-network denoted $\widetilde{Q}^*$\label{ntn:best_approx_Q} before overfitting (by selecting a discount factor lower than the maximum and by using early stopping). It also has the advantage of picking up the Q-network at an epoch less affected by instabilities.
The selected trained Q-network is then used in a test environment ($y=3$) to provide an independent estimation of how well the resulting policy performs.
Technical details relative to the deep Q-network algorithm used are given in Appendix \ref{app:microgrid}.

To empirically demonstrate the effect of the bias-overfitting tradeoff, we artificially reduce the diversity of the training and validation time-series by a factor $\kappa=\{1,2,4,8,16\}$.
Overall, the time series for training and validation should still be as close to 365 days as possible (the epochs are run on approximately the same given number of time steps for all cases). 
In addition, the true underlying processes should be respected as closely as possible by (i)~guaranteeing that the succession of the seasons are not corrupted and (ii)~guaranteeing that consecutive days in the original time series should be kept consecutive whenever possible.

In order to do so, the time series of one year are divided into four seasons. Each season is then split into $\kappa$ blocs of the same number of days.
For each season, data reduction is done by replicating one of the blocs instead on all remaining blocs of the same season. This artificial reduction on the available data is illustrated on Figure~\ref{POMDP_th1}.
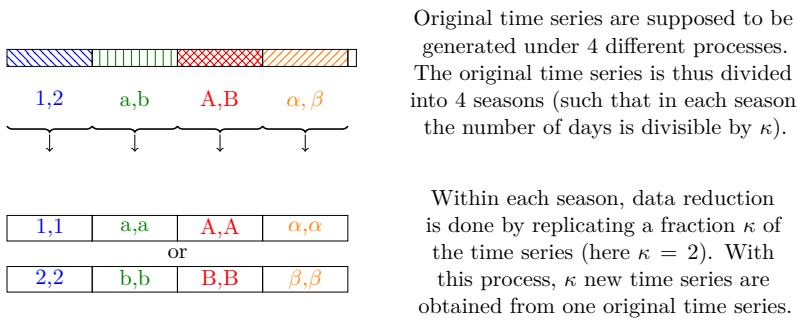
\begin{figure}[ht!]
 \centering
    \resizebox{0.75\textwidth}{!}{%
\begin{tikzpicture}[scale=1.5]

	\draw[pattern=north west lines, pattern color=blue] (0,0.1) rectangle (1,-0.1);
	\draw[pattern=vertical lines, pattern color=green!50!black] (1,0.1) rectangle (2,-0.1);
	\draw[pattern=crosshatch, pattern color=red] (2,0.1) rectangle (3,-0.1);
	\draw[pattern=north east lines, pattern color=orange] (3,0.1) rectangle (4,-0.1);
	\draw[] (4,0.1) rectangle (4.1,-0.1);

	\node[align=center, text width=3cm, color=blue] at (0+0.5,-0.5) {1,2};
	\node[align=center, text width=3cm, color=green!50!black] at (1+0.5,-0.5) {a,b};
	\node[align=center, text width=3cm, color=red] at (2+0.5,-0.5) {A,B};
	\node[align=center, text width=3cm, color=orange] at (3+0.5,-0.5) {$\alpha,\beta$};

		\foreach \x in {0,...,3}{
			\draw [thick,decorate,decoration={brace,amplitude=3pt,raise=0pt,mirror}] (\x,-0.8) -- (\x+1,-0.8);
			\draw [->] (\x+0.5,-0.9) -- (\x+0.5,-1.1);
		}

%
%

		\node[align=center, text width=3cm, color=blue] at (0.5,-1.4-0.6) {1,1};
		\node[align=center, text width=3cm, color=green!50!black] at (1+0.5,-1.4-0.6) {a,a};
		\node[align=center, text width=3cm, color=red] at (2+0.5,-1.4-0.6) {A,A};
		\node[align=center, text width=3cm, color=orange] at (3+0.5,-1.4-0.6) {$\alpha$,$\alpha$};

		\node[align=center, text width=3cm, color=blue] at (0.5,-1.4-2*0.6) {2,2};
		\node[align=center, text width=3cm, color=green!50!black] at (1+0.5,-1.4-2*0.6) {b,b};
		\node[align=center, text width=3cm, color=red] at (2+0.5,-1.4-2*0.6) {B,B};
		\node[align=center, text width=3cm, color=orange] at (3+0.5,-1.4-2*0.6) {$\beta$,$\beta$};

		\node[align=center, text width=3cm] at (2.,-2.3) {or};

		\foreach \s in {0,...,3}{
			\draw[] (0+\s,-1.85) rectangle (1+\s,-2.15);
			\draw[] (0+\s,-2.45) rectangle (1+\s,-2.75);
		}


\node[align=center, text width=7cm] at (7,-0.2) {Original time series are supposed to be generated under 4 different processes. The original time series is thus divided into 4 seasons (such that in each season the number of days is divisible by $\kappa$).};
\node[align=center, text width=7cm] at (7,-2.3) {Within each season, data reduction is done by replicating a fraction $\kappa$ of the time series (here $\kappa=2$). With this process, $\kappa$ new time series are obtained from one original time series.};

\end{tikzpicture}
}
\caption{Illustration of how the quantity of data in a time series of one year is artificially reduced by a factor $\kappa=2$ while guaranteeing that the specificity of the seasons are essentially preserved. This schema applies for both the consumption and the production time-series and for both the training time series and validation time series.}
\label{POMDP_th1}
\end{figure}

\subsubsection{Results of the Experiments}
The purpose of the experiments is to illustrate the theoretical results, and are given in Figure~\ref{fig:microgrid_benchmark}. On the one hand, using a long history may pose problems relative to overfitting when the amount of training data is limited. On the other hand, using a long history of observations allows achieving a good operation when enough data is available (information on the latent features are preserved). 

\begin{figure}[!ht]
    \centering
        \includegraphics[width=220pt]{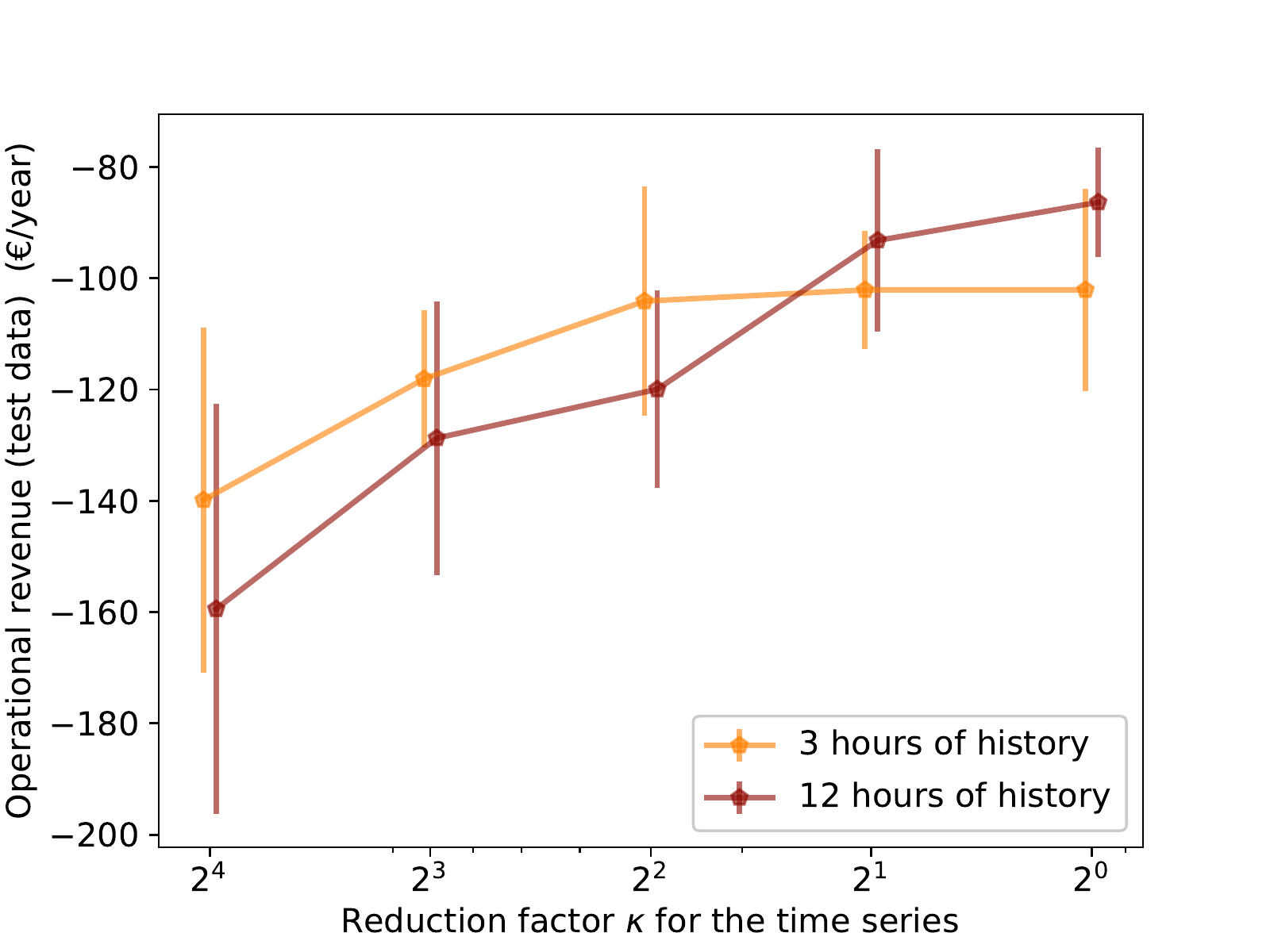}
        \caption{Evolution of the performance of the policy in the test setting (one year of unseen past time series for both the consumption and production). Using a relatively short history of observations is beneficial when the amount of data is limited while a longer history of observations leads to improved results when the amount of data increases (lower $\kappa$ means that more data is available). For each point, the mean and the standard deviation are calculated over $2^3$ different runs of the Deep Q-network algorithm (over different ways to split the time series defined in appendix and for different random seeds).}
        \label{fig:microgrid_benchmark}
\end{figure}

As already noted, there is a close link between (i)~choosing a function approximator and (ii)~selecting features, since the function approximator characterizes how the features will be treated into higher levels of abstraction (\textit{a fortiori} it can thus give more or less weight to some features). Notwithstanding, this example shows that, in general, the use of function approximators may still benefit from an efficient feature selection process. Indeed, we still observe a tradeoff between bias and overfitting in this microgrid application, even with the convolutional architecture of the deep Q-network. 

\section{Conclusion and Future Works}
\label{conclusion}

This paper discusses the bias-overfitting tradeoff of batch RL algorithms in the context of POMDPs. 
Most current RL works compare algorithms in an online setting, where the performance depends on many factors such as improved exploration, better generalization, or better numerical efficiency of the gradient step (in the case of deep RL). In contrast, this paper is focused on one source of sub-optimality related to generalization from limited data in the case of POMDPs.
In that context, we propose an analysis showing that, similarly to supervised learning techniques, RL may face a bias-overfitting dilemma in situations where the policy class is too large compared to the batch of data.
In such situations, we show that it may be preferable to concede an asymptotic bias in order to reduce overfitting.
This (favorable) asymptotic bias may be introduced through different manners: (i)~downsizing the state representation, (ii)~using specific types of function approximators and (iii)~lowering the discount factor.

The main theoretical results of this paper relate to the state representation; the originality of the setting proposed in this paper compared to \citeauthor{maillard2011selecting}~\citeyear{maillard2011selecting}, \citeauthor{ortner2014selecting}~\citeyear{ortner2014selecting} and the related work is mainly to formalize the problem in a batch setting (limited set of tuples) instead of the online setting, where they investigate the E/E dilemma. As compared to \citeauthor{jiang2015dependence}~\citeyear{jiang2015dependence}, the originality is to consider a partially observable setting. We introduce the notion of the $\epsilon$-sufficient statistics and showed that it allows us to formalize the intuition of the bias-overfitting trade-off in a rigorous way and provides new insights compared to the MDP case. In particular, the bound of Theorem \ref{theorem_bias} is a new result based on L1 error terms of the associated belief states. 
There are also interesting insights from the techniques used to derive that bound, where we use the formalism of bisimulation metrics and the data processing inequality.

The work proposed in this paper may also be of interest in online settings because, at each stage, obtaining a performant policy from given data is part of the solution to an efficient exploration/exploitation tradeoff. 
For instance, optimizing the bias-overfitting tradeoff suggests that it can be beneficial to dynamically adapt the feature space and the function approximator. This can be done through ad hoc regularization or by adapting the neural network architecture, using for instance the NET2NET transformation \cite{chen2015net2net}.


\subsection*{Acknowledgments}
We thank the anonymous reviewers for many helpful comments.
The authors thank the Walloon Region (Belgium) that funded this research in the context of the BATWAL project.
We also acknowledge financial support by the Natural Science and Engineering Research Council Canada (NSERC) and Samsung Electronics Co Ltd. Guillaume Rabusseau acknowledges support of an IVADO postdoctoral fellowship.

\bibliographystyle{theapa}
\bibliography{references}{} 

\newpage
\appendix
\section{Proofs}

\subsection{Proof of Proposition \ref{main_proposition}}
\label{main_proof}

\paragraph{Proposition \ref{main_proposition}}

Let $\phi_\epsilon$ be an $\epsilon$-sufficient mapping, and let $\phi_0$ be a sufficient mapping.
Then, for any $H^{(1)}, H^{(2)}$ such that $\phi_\epsilon(H^{(1)})=\phi_\epsilon(H^{(2)})$, we have
\begin{align*}
\underset{a}{\max} & \left|\mathcal Q^{\pi_{D_\infty, \phi_0}}_{\hat M_{D_\infty, \phi_0}}(\phi_0(H^{(1)}),a) -\mathcal Q^{\pi_{D_\infty, \phi_0}}_{\hat M_{D_\infty, \phi_0}}(\phi_0(H^{(2)}),a)\right| \quad \le \epsilon \frac{R_{max}}{(1-\gamma)}. \\
\end{align*}

\begin{proof}

\newcommand{\val}{\mathcal V}
\newcommand{\Qval}{\mathcal Q}
\newcommand{\QvalFreqMDP}[1]{\Qval^{#1}_{\hat M_{D_\infty, \phi_0}}}
\newcommand{\QvalOptFreqMDP}{\Qval^{\pi_{D_\infty, \phi_0}}_{\hat M_{D_\infty, \phi_0}}}
\newcommand{\valFreqMDP}[1]{\val^{#1}_{\hat M_{D_\infty, \phi_0}}}
\newcommand{\valOptFreqMDP}{\val^{\pi_{D_\infty, \phi_0}}_{\hat M_{D_\infty, \phi_0}}}
\newcommand{\Rmax}{R_{\mathit{max}}}
\newcommand{\esp}{\mathbb{E}}


%
%
%

We first define the notion of a policy conditioned on a previous history. 
Let $\pi:\phi_0(\mathcal H) \to \mathcal A$. 
Given a history $H\in\mathcal{H}$ we define $\pi_H : (\mathcal{A}\times\Omega)^* \to \mathcal{A}$
by $\pi_H( a_1 \omega_1\cdots a_k\omega_k ) = \pi(\phi_0(H a_1\omega_1 \cdots a_k \omega_k ))$, i.e. $\pi_H$ is the
policy $\pi$ conditioned on having previously observed the history $H$ \footnote{To simplify notations, one can consider that the rewards are included in the observations $\omega$.}.

Given a history $H\in\mathcal{H}$ 
we can then define the
Q-value of a state $s\in\mathcal S$ and an action $a\in\mathcal A$ in the underlying MDP following $\pi$
and assuming that $H$ has been observed:
\begin{align*}
\Qval^{\pi_H}(s,a)=&R'(s,a)\\
&+ \gamma \esp_{s_1\mid s,a}  \esp_{\omega_1\mid s_1}  R'(s_1,a_1=\pi_H(a\omega_1))\\
&+ \gamma^2 \esp_{s_1\mid s,a}  \esp_{\omega_1\mid s_1} \esp_{s_2\mid s_1,a_1}  \esp_{\omega_2\mid s_2} R'(s_2,\pi_H(a\omega_1 a_1\omega_2))\\
&+\ \cdots
\end{align*}
where $R'(s,a) = \esp_{s'\mid s,a} R(s,a,s')$.
Note that $\QvalFreqMDP{\pi}(\phi_0(H),a) = \mathbb{E}_{s\sim b_{\phi_0}(\cdot | \phi_0(H))} \Qval^{\pi_H}(s,a)$ for any history $H\in\mathcal{H}$ and
any action $a\in\mathcal A$  since
$\phi_0$ is a sufficient mapping.

We can now prove the proposition.
Let $a\in\mathcal A$.
Let $\pi^*= \pi_{D_\infty, \phi_0}$ be the optimal policy in $\hat{M}$ and
let $\pi_i = \pi^*_{H_i}$ be the optimal policy conditioned on having observed $H_i$ for $i=1,2$. 
%
Since $\pi^*$ is optimal we have
\begin{align*}
\Qval^{\pi^*}(\phi_0(H^{(1)}),a) & =  \mathbb{E}_{s\sim b_{\phi_0}(\cdot | \phi_0(H_1))} \Qval^{\pi_1}(s,a) \\
& \geq  \mathbb{E}_{s\sim b_{\phi_0}(\cdot | \phi_0(H_1))} \Qval^{\pi_2}(s,a)
\end{align*}
and
\begin{align*}
\Qval^{\pi^*}(\phi_0(H^{(2)}),a) & =  \mathbb{E}_{s\sim b_{\phi_0}(\cdot | \phi_0(H_2))} \Qval^{\pi_2}(s,a) \\
& \geq  \mathbb{E}_{s\sim b_{\phi_0}(\cdot | \phi_0(H_2))} \Qval^{\pi_1}(s,a).
\end{align*}
We then have
\begin{align*}
\QvalFreqMDP{\pi^*} (\phi_0(H^{(1)}),a) & - \QvalFreqMDP{\pi^*}(\phi_0(H^{(2)}),a)\\
&=
\mathbb{E}_{s\sim b_{\phi_0}(\cdot | \phi_0(H_1))} \Qval^{\pi_1}(s,a) - \mathbb{E}_{s\sim b_{\phi_0}(\cdot | \phi_0(H_2))} \Qval^{\pi_2}(s,a)\\
& \le
 \mathbb{E}_{s\sim b_{\phi_0}(\cdot | \phi_0(H_1))} \Qval^{\pi_1}(s,a) - \mathbb{E}_{s\sim b_{\phi_0}(\cdot | \phi_0(H_2))} \Qval^{\pi_1}(s,a)  \\
& =
\sum_{s\in\mathcal S} \left(b(s | H^{(1)}) -  b(s | H^{(2)})\right) \Qval^{\pi_1}(s,a)  \\
& \leq
\frac{\epsilon}{1-\gamma} \Rmax,
\end{align*}
where we used the facts that $b(\cdot | H^{(1)}) - b(\cdot | H^{(2)})$ is a vector of $L_1$-norm less than
$2\epsilon$~(by Lemma~\ref{dist_epsilon}) whose components sum to zero and $|\Qval^{\pi_1}(s,a) -\Qval^{\pi_1}(s',a)| \leq \frac{\Rmax}{1-\gamma}$ for any $s,s'\in\mathcal S$.

Applying the same argument \emph{mutatis mutandis}
we obtain $\QvalFreqMDP{\pi^*}  (\phi_0(H^{(2)}),a) - \QvalFreqMDP{\pi^*}(\phi_0(H^{(1)}),a) \leq \frac{\epsilon}{1-\gamma} \Rmax$
from which the result follows.

\end{proof}

\begin{proof}[Alternative proof]
An alternative proof of independent interest makes use of the formalism of the bisimulation metric \cite{ferns2004metrics} along with the data processing inequality. 
Let us consider $d \in \mathcal M$, where $\mathcal M$ is the set of all semi-metrics on $\Sigma_0$ with distance at most 1. We fix a particular $d$ as follows:
$\forall \sigma^{(1)}, \sigma^{(2)} \in \Sigma_0 $:
\begin{equation*}
\begin{split}
d(\sigma^{(1)},\sigma^{(2)}) & =\underset{a}{\max} \left( c_R\left\lvert \hat R'\left(\sigma^{(1)},a\right) - \hat R'\left(\sigma^{(2)},a\right)\right\lvert \right. \\
& \left. + c_T d_P \big(\hat T(\sigma^{(1)},a,\cdot), \hat T(\sigma^{(2)},a,\cdot)\big) \right),
\end{split}
\end{equation*}
where $d_P$ is some probability metric, and 
where $c_R,c_T\ge0$ are such that $c_R+c_T\le1$.
We define $F : \mathcal M  \rightarrow \mathcal M$ by
\begin{equation*}
\begin{split}
F(d)(\sigma^{(1)},\sigma^{(2)}) & =\underset{a}{\max} \left( c_R \left\lvert \hat R'\left(\sigma^{(1)},a\right)- \hat R'\left(\sigma^{(2)},a\right)\right\lvert \right. \\
& \left. + c_T T_K(d) \big(\hat T(\sigma^{(1)},a,\cdot), \hat T(\sigma^{(2)},a,\cdot)\big) \right).\\
\end{split}
\end{equation*}
where $T_K(d)$ is the Kantorovich metric induced by $d$.
From Lemma \ref{bisimulation_fixed}, $F$ has a least fixed-point $d_{fix}$, and it is a bisimulation metric.
From Lemmas \ref{lemmaR}, \ref{lemmaT} (using the data processing inequality), we have $\forall \sigma^{(1)}, \sigma^{(2)} \in \Sigma_0$ that
$$d_{fix}(\sigma^{(1)},\sigma^{(2)}) \le \frac{1}{2} \lVert b(\cdot \mid \sigma^{(1)})-b(\cdot \mid \sigma^{(2)}) \lVert_1 R_{max}.$$ 
Using Lemma \ref{bisimulation_Q}, it follows that
\begin{equation*}
\begin{split}
& \underset{a}{\max} \left|\mathcal Q^{\pi_{D_\infty, \phi_0}}_{\hat M_{D_\infty, \phi_0}}(\phi_0(H^{(1)}),a) -\mathcal Q^{\pi_{D_\infty, \phi_0}}_{\hat M_{D_\infty, \phi_0}}(\phi_0(H^{(2)}),a)\right| \le \epsilon \frac{R_{max}}{(1-\gamma)}. \\
\end{split}
\end{equation*}

\end{proof}

\begin{lemma}
\label{dist_epsilon}
Let
$H^{(1)}, H^{(2)} \in H$ be two histories such that $\phi_\epsilon(H^{(1)}) =
\phi_\epsilon(H^{(2)})$. Then,
$$\lVert b(\cdot| H^{(1)}) - b(\cdot | H^{(2)}) \lVert_1 \le 2 \epsilon.$$
\end{lemma}
\begin{proof}
\begin{equation*}
\begin{split}
\left\lVert b(\cdot \mid H^{(1)}) - \right. \left. b(\cdot \mid H^{(2)}) \right\lVert_1 = & \left\lVert  b(\cdot \mid H^{(1)}) - b_{\phi_\epsilon}\left(\cdot \mid \phi_\epsilon(H^{(1)})\right) \right. \\
& \left. + b_{\phi_\epsilon}\left(\cdot \mid \phi_\epsilon(H^{(2)})\right) - b(\cdot \mid H^{(2)})  \right\lVert_1 \\
\le & \left\lVert  b(\cdot \mid H^{(1)}) - b_{\phi_\epsilon}\left(\cdot \mid \phi_\epsilon(H^{(1)})\right)  \right\lVert_1\\
& + \left\lVert b_{\phi_\epsilon}\left(\cdot \mid \phi_\epsilon(H^{(2)})\right) - b(\cdot \mid H^{(2)})  \right\lVert_1 \\
\le & 2 \epsilon,
\end{split}
\end{equation*}

\end{proof}

\begin{lemma}
\label{bisimulation_fixed}
Let $c_R,c_T\ge0$ and $c_R+c_T\le1$. Define $F : \mathcal M  \rightarrow \mathcal M$ by
\begin{equation*}
\begin{split}
& F(d)(\sigma^{(1)},\sigma^{(2)})=\underset{a}{\max} \left( c_R \left\lvert \hat R'\left(\sigma^{(1)},a\right) - \hat R'\left(\sigma^{(2)},a\right)\right\lvert \right.\\
& \left. + c_T T_K(d)\big(\hat T(\sigma^{(1)},a,\cdot), \hat T(\sigma^{(2)},a,\cdot)\big) \right) .\\
\end{split}
\end{equation*}
Then F has a least fixed-point, $d_{fix}$, and $d_{fix}$ is a bisimulation metric.
\end{lemma}          
\begin{proof}
See proof of Theorem 4.5 by \citeauthor{ferns2004metrics}~\citeyear{ferns2004metrics}.
\end{proof}

\begin{lemma}
\label{lemmaR}
Let $H^{(1)}, H^{(2)} \in \mathcal H$ such that
$$\lVert b(\cdot \mid \phi_0(H^{(1)})-b(\cdot \mid \phi_0(H^{(2)}) \lVert_1 \le 2 \epsilon.$$ 
Then
$$\underset{a}{\max} \left| \hat R'\left(\phi_0(H^{(1)}),a\right)-\hat R'\left(\phi_0(H^{(2)}),a\right) \right| \le \epsilon R_{max}.$$
\end{lemma}
\begin{proof}
This follows directly from the fact that $\forall H$:
$$\hat R'(\phi_0(H),a)=\sum_{s \in \mathcal S}\mathbb P(s|\phi_0(H)) R'(s,a),$$
where  $R'(s,a)=\mathbb E_{s'\mid s,a}R(s,a,s')$.
\end{proof}

\begin{lemma}
\label{lemmaT}
Let $H^{(1)}, H^{(2)} \in \mathcal H$ such that
$$\left\lVert b(\cdot \mid \phi_0(H^{(1)})-b(\cdot \mid \phi_0(H^{(2)}) \right\lVert_1 \le 2 \epsilon.$$ 
Then
\begin{equation*}
\begin{split}
\sum_{\omega \in \Omega} & \left\lVert \mathbb P\left(\omega \mid \phi_0(H^{(1)}),a\right) b(\cdot \mid \phi_0(H^{(1)}),a,\omega) \right. \\
& \left. - \mathbb P\left(\omega \mid \phi_0(H^{(2)}),a\right) b(\cdot \mid \phi_0(H^{(2)}),a,\omega) \right\lVert_1 \le 2\epsilon.
\end{split}
\end{equation*}
\end{lemma}
\begin{proof}
For any $a \in \mathcal A$ by direct application of the data processing inequality (DPI) in the special case of the total variation:
$$\left\lVert b\left(.|\phi_0(H^{(1)}),a\right) - b\left(.|\phi_0(H^{(2)}),a\right) \right\lVert_1 \le 2\epsilon.$$

Moreover, as illustrated on Figure \ref{fig:DPI} and once more by application of the DPI in the context of the special case of the total variation, we have that
\begin{equation*}
\begin{split}
& \frac{1}{2} \mathbb E_{Q_Y} \left| \frac{P_Y}{Q_Y} -1 \right| \le \frac{1}{2} \mathbb E_{Q_X} \left| \frac{P_X}{Q_X} -1 \right|.
\end{split}
\end{equation*}
with the distributions
\begin{itemize}
\item $P_X=\mathbb P\left(s|\phi_0(H^{(1)}),a\right)$, 
\item $Q_X=\mathbb P\left(s|\phi_0(H^{(2)}),a\right)$, 
\item $P_Y=\mathbb P\left(\omega,s \mid \phi_0(H^{(1)}),a\right)$, and
\item $Q_Y=\mathbb P\left(\omega,s \mid \phi_0(H^{(2)}),a\right)$.
\end{itemize}
Note that the output distributions is on the tuple $(\omega,s)$.

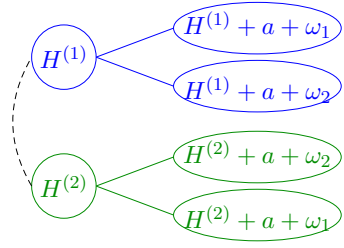
\begin{figure}[ht!]
\centering
\resizebox{0.3\textwidth}{!}{%
\begin{tikzpicture}[
dot/.style    = {anchor=base,fill,circle,minimum size=0.1cm, inner sep=0.2pt}]

\node [xshift=0cm, blue] (dot1_belief) at (0,1) {$H^{(1)}$};
\node [xshift=0cm, black!45!green] (dot2_belief) at (0,-1) {$H^{(2)}$};

\node [xshift=0cm, blue] (dot1_belief) at (3,1.5) {$H^{(1)}+a+\omega_1$};
\node [xshift=0cm, blue] (dot2_belief) at (3,0.5) {$H^{(1)}+a+\omega_2$};

\node [xshift=0cm, black!45!green] (dot1_belief) at (3,-1.5) {$H^{(2)}+a+\omega_1$};
\node [xshift=0cm, black!45!green] (dot2_belief) at (3,-0.5) {$H^{(2)}+a+\omega_2$};

\draw [blue](0,1) circle (0.5cm);
\draw [black!45!green](0,-1) circle (0.5cm);
\draw [blue](3,1.45) ellipse (1.3cm and 0.4cm);
\draw [blue](3,0.55) ellipse (1.3cm and 0.4cm);
\draw [black!45!green](3,-1.45) ellipse (1.3cm and 0.4cm);
\draw [black!45!green](3,-0.55) ellipse (1.3cm and 0.4cm);

\draw[blue] (.5,1) -- (1.7,1.45);
\draw[blue] (.5,1) -- (1.7,0.55);
\draw[black!45!green] (.5,-1) -- (1.7,-1.45);
\draw[black!45!green] (.5,-1) -- (1.7,-0.55);

\draw[densely dashed]    (-0.5,-1) to[out=120,in=-120] (-0.5,1);

\end{tikzpicture}
}
\caption{Illustration of the two histories $H^{(1)}$ and $H^{(2)}$. Here we illustrate only for one given action $a$ and two possible observations $\omega_1$ and $\omega_2$. The dashed line illustrates the key relation between $H^{(1)}$ and $H^{(2)}$, which is $\left\lVert b(\cdot \mid H^{(1)}) - b(\cdot \mid H^{(2)}) \right\lVert_1 \le 2 \epsilon$. In expectation, this relation is preserved through the action and observations.}
\label{fig:DPI}
\end{figure}

Since the right hand-side in the aforementioned DPI can be bounded by $\epsilon$:
\begin{equation*}
\begin{split}
& \frac{1}{2} \sum_{s \in \mathcal S} \mathbb P\left(s|\phi_0(H^{(2)}),a\right) \left| \frac{\mathbb P\left(s|\phi_0(H^{(1)}),a\right)}{\mathbb P\left(s|\phi_0(H^{(2)}),a\right)}-1 \right| \le \epsilon,
\end{split}
\end{equation*}
we have
\begin{equation*}
\begin{split}
\frac{1}{2} \sum_{\omega \in \Omega} \sum_{s \in \mathcal S} & \left| \mathbb P\left(\omega \mid \phi_0(H^{(1)}),a\right) \mathbb P\left(s \mid \phi_0(H^{(1)}),a,\omega \right) \right. \\
& \left. - \mathbb P\left(\omega \mid \phi_0(H^{(2)}),a\right) \mathbb P\left(s \mid \phi_0(H^{(2)}),a,\omega \right) \right| \le \epsilon.
\end{split}
\end{equation*}
This can be rewritten as:
\begin{equation*}
\begin{split}
\sum_{\omega \in \Omega} & \left\lVert \mathbb P\left(\omega \mid \phi_0(H^{(1)}),a\right) b(\cdot \mid \phi_0(H^{(1)}),a,\omega) \right. \\
& \left. - \mathbb P\left(\omega \mid \phi_0(H^{(2)}),a\right) b(\cdot \mid \phi_0(H^{(2)}),a,\omega) \right\lVert_1 \le 2\epsilon.
\end{split}
\end{equation*}
\end{proof}

\begin{lemma}
\label{bisimulation_Q}
Suppose $\gamma \le c_T$. Then $\forall \sigma^{(1)}, \sigma^{(2)} \in \Sigma_0$:
$$c_R \max_{a} \left| \mathcal Q^*(\sigma^{(1)},a)-\mathcal Q^*(\sigma^{(2)},a) \right| \le d_{fix}(\sigma^{(1)},\sigma^{(2)})$$

where $\mathcal Q^*$ stands for $\mathcal Q^{\pi_{D_\infty, \phi_0}}_{\hat M_{D_\infty, \phi_0}}$.
\end{lemma}

\begin{proof}
The proof is done using the same techniques than the proof of Theorem 5.1 by \citeauthor{ferns2004metrics}~\citeyear{ferns2004metrics}.
Let us first prove that
$$\max_{a} c_R\left| \mathcal Q_n(\sigma^{(1)},a)-\mathcal Q_n(\sigma^{(2)},a) \right| \le d_n(\sigma^{(1)},\sigma^{(2)})$$
where 
$F(d_{n}(\sigma^{(1)},\sigma^{(2)})) = d_{n+1}(\sigma^{(1)},\sigma^{(2)})$, and
\begin{equation*}
\begin{split}
\mathcal Q_{n+1} (\phi_0(H),a)&=\hat R'(\phi_0(H),a) +\gamma \sum_{\varphi \in \phi_0(\mathcal H)} \hat T(\phi_0(H),a,\varphi) \mathcal V_{n}(\varphi),
\end{split}
\end{equation*}
with $\mathcal V_{n}(\varphi)=\underset{a'}{\max} \mathcal Q_{n} (\varphi,a')$. The result would then follow by taking the limit.

$\forall H^{(1)}, H^{(2)} \in \mathcal H$: $\phi_\epsilon(H^{(1)})=\phi_\epsilon(H^{(2)})$, it follows that
\begin{small}
\begin{equation*}
\begin{split}
c_R ~ \underset{a}{\max} & \left|\mathcal Q_{n+1}\left(\phi_0(H^{(1)}),a\right)-\mathcal Q_{n+1}\left(\phi_0(H^{(2)}),a\right)\right| \\
& \le c_R \underset{a}{\max} \Bigg( \left| \hat R'\left(\phi_0(H^{(1)}),a\right)- \hat R'\left(\phi_0(H^{(2)}),a\right)\right|\\
& + \gamma \left| \hat T(\phi_0(H^{(1)}),a,\cdot) \mathcal V_n\left(\cdot)\right) \right. - \left. \hat T(\phi_0(H^{(1)}),a,\cdot) \mathcal V_n \left(\cdot)\right) \right| \Bigg)\\
& \le \underset{a}{\max} \Bigg( c_R \left| \hat R'\left(\phi_0(H^{(1)}),a\right)- \hat R'\left(\phi_0(H^{(2)}),a\right)\right|\\
& + c_T \left| \left( \hat T(\phi_0(H^{(1)}),a,\cdot) - \hat T(\phi_0(H^{(2)}),a,\cdot) \right) \frac{c_R \gamma}{c_T} \right. \left. \mathcal V_n(\cdot) \right| \Bigg)\\
& \le \underset{a}{\max} \Bigg( c_R \left| \hat R'\left(\phi_0(H^{(1)}),a\right)- \hat R'\left(\phi_0(H^{(2)}),a\right)\right|\\
& + c_T T_K(d_n) \left( \hat T(\phi_0(H^{(1)}),a,\cdot) - \hat T(\phi_0(H^{(2)}),a,\cdot) \right)\Bigg) \\
& = F\left(d_{n}\big(\phi_0(H^{(1)}),\phi_0(H^{(2)})\big)\right) \\
& = d_{n+1}\left(\phi_0(H^{(1)}),\phi_0(H^{(2)})\right),
\end{split}
\end{equation*}
\end{small}
where we have used the fact that $\{\frac{c_R \gamma}{c_T} \mathcal V_n(u):u \in \Sigma\}$ is a feasible solution to the primal LP defined by the Kantorovich distance.
\end{proof}

\subsection{Proof of Theorem \ref{union_bound}} 
\label{app:union_bound}
\paragraph{Sketch of the proof}
The idea of the proof is first to bound the difference between value functions following different policies by a bound between value functions estimated in different environments but following the same policy (more precisely a max over a set of policies of such a bound). Once that is done, a bound in probability using Hoeffding's inequality can be obtained.

Let us denote $\varphi \in \phi(\mathcal H)$:
\begin{equation*}
\begin{split}
V_M^{\pi_{D_\infty, \phi}} (\varphi) -V_M^{\pi_{D, \phi}}(\varphi) = & (V_M^{\pi_{D_\infty, \phi}}(\varphi) -\mathcal V_{\hat M_{D}}^{\pi_{D_\infty, \phi}} (\varphi) )
 - (V_M^{\pi_{D, \phi}}(\varphi) - \mathcal V_{\hat M_{D}}^{\pi_{D, \phi}} (\varphi) ) \\
 & + (\mathcal V_{\hat M_{D}}^{\pi_{D_\infty, \phi}}(\varphi) - \mathcal V_{\hat M_{D}}^{\pi_{D, \phi}}(\varphi) ) \\
\le & (V_M^{\pi_{D_\infty, \phi}}(\varphi) - \mathcal V_{\hat M_{D}}^{\pi_{D_\infty, \phi}}(\varphi) ) - (V_M^{\pi_{D, \phi}}(\varphi) - \mathcal V_{\hat M_{D}}^{\pi_{D, \phi}}(\varphi) ) \\
\le & \ 2 \underset{\pi \in \{\pi_{D_\infty, \phi},\pi_{D, \phi}\} }{\max} \left| V_{M}^{\pi}(\varphi) - \mathcal V_{\hat M_{D}}^{\pi}(\varphi) \right|. \\
\end{split}
\end{equation*}

It follows that $\forall H$
\begin{small}
\begin{equation*}
\begin{split}
V_M^{\pi_{D_\infty, \phi}} (\varphi) -V_M^{\pi_{D, \phi}} (\varphi) \le & \ 2 \underset{ \pi \in \{\pi_{D_\infty, \phi},\pi_{D, \phi}\} \ }{\max} 
 \underset{\varphi \in \phi(\mathcal H)}{\max} \left| Q_{M}^{\pi}(\varphi,\pi(\varphi)) - \mathcal Q_{\hat M_{D}}^{\pi}(\varphi,\pi(\varphi)) \right|  \\
\le & \ 2 \underset{ \pi \in \{\pi_{D_\infty, \phi},\pi_{D, \phi}\} \ }{\max}
 \underset{\varphi \in \phi(\mathcal H),a \in \mathrm A
 }{\max} \left| Q_{M}^{\pi}(\varphi,a) - \mathcal Q_{\hat M_{D}}^{\pi}(\varphi,a) \right|,  \\
\end{split}
\end{equation*}
\end{small}
where $Q_{M}^{\pi}(\varphi,a)$ is the action-value function for policy $\pi$ in $M$ with $a \in \mathcal A$. 
Similarly $\mathcal Q^\pi_{\hat M}(\varphi,a)$ is the action-value function for policy $\pi$ in $\hat M$.

By Lemma \ref{partial}, we have:
\begin{equation*}
\begin{split}
V_M^{\pi_{D_\infty, \phi}} (\varphi) -V_M^{\pi_{D, \phi}}(\varphi) & \qquad \le \frac{2}{1-\gamma} \underset{ \pi \in \{\pi_{D_\infty, \phi},\pi_{D, \phi}\} \ }{\max} \underset{\varphi \in \phi(\mathcal H), a \in A
}{\max} \\
& \left| \hat{R'}(\varphi,a) + \gamma \sum_{\varphi' \in \phi(\mathcal H)} \hat{T}(\varphi,a,\varphi') V_{M}^\pi(\varphi') - Q_{M}^{\pi}(\varphi,a) \right|,  \\
\end{split}
\label{key_equ}
\end{equation*}
where $\hat R'(\phi_0(H),a)=\sum_{\varphi \in \phi_0(\mathcal H)} \hat T(\phi_0(H),a,\varphi) \hat R(\phi_0(H),a, \varphi)$.

With $\mathcal R \in [0,R_{max}]$, we notice that $\hat{R'}(\varphi,a) + \gamma \sum_{\varphi' \in \phi(\mathcal H)} \hat{T}(\varphi,a,\varphi') V_{M}^\pi(\varphi')$ 
is the mean of i.i.d. variables bounded in the interval [0,$\frac{R_{max}}{1-\gamma}$] 
and with mean $Q_{M}^{\pi}(\varphi,a)$ for any policy $\pi: \phi(\mathcal H) \rightarrow \mathrm A$.
Therefore, 
according to Hoeffding's inequality \cite{hoeffding1963probability}, we have with $n$ the number of tuples for every pair $(\varphi,a)$:
\begin{small}
\begin{equation}
  \begin{aligned}
\mathbb{P} \Bigg\{ \left| \hat{R'}(\varphi,a) + \gamma \sum_{\varphi' \in \phi(\mathcal H)} \hat{T}(\varphi,a,\varphi') V_{M}^\pi(\varphi')
- Q_{M}^{\pi}(\varphi,a) \right|  > t \Bigg\} \\
\le 2 \exp \left( \frac{-2nt^2}{\left(R_{max}/(1-\gamma)\right)^2} \right).
  \end{aligned}
\label{eq:Hoeffding}
\end{equation}
\end{small}
 
As we want to obtain a bound over all pairs $(\varphi,a)$ and a union bound on all policies $\pi \in \Pi$ s.t. $|\Pi|=|\mathrm A|^{|\phi(\mathcal H)|}$ (indeed Equation \ref{eq:Hoeffding} does not hold for $\pi_{D, \phi}$ alone because that policy is not chosen randomly), we want the right-hand side of equation \ref{eq:Hoeffding} to be $\frac{\delta}{|\Pi| |\phi(\mathcal H)| |\mathrm A|}$. This gives $t(\delta)=\left(\frac{R_{max}}{1-\gamma}\right) \sqrt{\left(\frac{1}{2n}  \ln\left(\frac{2 |\Pi| |\phi(\mathcal H)| |\mathrm A|}{\delta}\right)\right)}$ and we conclude that:

\begin{equation*}
  \begin{split}
  \underset{\varphi \in \phi(\mathcal H)}{\max} & \left( V_M^{\pi_{D_\infty, \phi}}(\varphi)-V_M^{\pi_{D, \phi}}(\varphi)\right) \le \frac{2 R_{max}}{(1-\gamma)^2} \sqrt{ \frac{1}{2n}  ln\left(\frac{2 |\phi(\mathcal H)| |\mathrm A|^{1+|\phi(\mathcal H)|}}{\delta}\right)}
  \end{split}
  \end{equation*}
with probability at least $1-\delta$.

\begin{lemma}
\label{partial}
For any $M=(S,A,T,R,\Omega,O,\gamma)$ and the frequentist-based augmented MDP $\hat{M}=(\Sigma,\mathrm A,\hat T,\hat R,\Gamma)$ defined from $M$ according to definition \ref{augmented_DP}, we have
$\forall \pi: \phi(\mathcal H) \rightarrow A$:

\begin{equation*}
\begin{split}
\underset{\varphi \in \phi(\mathcal H),a \in \mathrm A
 }{\max} \left| Q_{M}^{\pi}(\varphi,a) - \mathcal Q_{\hat M_{D}}^{\pi}(\varphi,a) \right| & \le \frac{1}{1-\gamma} \underset{\varphi \in \phi(\mathcal H),a \in A}{\max} \\
& \left| \hat{R'}(\varphi,a) + \gamma \sum_{\varphi' \in \phi(\mathcal H)} \hat{T}(\varphi,a,\varphi') V_{M}^\pi(\varphi') - Q_{M}^{\pi}(\varphi,a) \right|.
\end{split}
\end{equation*}
\end{lemma}

\begin{proof}
Given any policy $\pi$, let us define $\mathcal Q_0, \mathcal Q_1, ..., \mathcal Q_m$ s.t.
$\mathcal Q_0(\varphi,a)=Q^\pi_M(\varphi,a)$ and \\
$\mathcal Q_m(\varphi, a) = \hat R'(\varphi,a) + \gamma \sum_{\varphi' \in \phi(\mathcal H)} \hat{T}(\varphi,a,\varphi') \mathcal V_{m-1}(\varphi')$, where $\mathcal V_{m-1}(\varphi) = \mathcal Q_{m-1}\left(\varphi, \pi(\varphi)\right)$. We have 
\begin{equation*}
\begin{split}
\left\lVert \mathcal Q_{m} -\mathcal Q_{m-1} \right\lVert_\infty & \le \gamma \underset{\varphi \in \phi(\mathcal H), a \in \mathrm A}{\max} \left| \sum_{\varphi' \in \phi(\mathcal H)} \hat{T}(\varphi,a,\varphi') (\mathcal V_{m-1}-\mathcal V_{m-2})(\varphi')\right| \\
& \le \gamma \left\lVert(\mathcal V_{m-1}-\mathcal V_{m-2})(\varphi)\right\lVert_\infty \\
& \le \gamma \left\lVert(\mathcal Q_{m-1}-\mathcal Q_{m-2})(\varphi,a)\right\lVert_\infty.
\end{split}
\end{equation*}

Taking the limit of $m \rightarrow \infty, \mathcal Q_m \rightarrow \mathcal Q^\pi_{\hat M}$, we have 
$$\left\lVert\mathcal Q^\pi_{\hat M_D} - Q^\pi_M\right\lVert_\infty \le \frac{1}{1-\gamma}\left\lVert Q_1-Q_0\right\lVert_\infty$$
which completes the proof.
\end{proof}

\subsection{Q-learning with Neural Network as a Function Approximator: Technical Details of Figure~\ref{fig:Random_POMDP_NN}} 
\label{app:Q_learn}
The neural network is made up of three intermediate fully connected layers with 20, 50 and 20 neurons with ReLu activation function and is trained with Q-learning. Weights are initialized with a Glorot uniform initializer \cite{glorot2010understanding}. It is trained using a target Q-network with a freeze interval (see \cite{mnih2015human}) of 100 mini-batch gradient descent steps. It uses an RMSprop update rule (learning rate of $0.005$, $\rho=0.9$), mini-batches of size 32 and 20000 mini-batch gradient descent steps.

\section{Algorithmic details}
\subsection{Microgrid Benchmark}
\label{app:microgrid}
\subsubsection{Neural Network Architecture}
\label{infosNN}
The inputs of the neural network architecture are provided by $\phi(H_t)$ (normalized into [0,1]), and the outputs represent the Q-values for each discretized action. 

The neural network processes the time series (when $h_c>2$ and $h_p>2$) thanks to a set of convolutions with 16 filters of $2 \times 1$ with stride 1, followed by a convolution with 16 filters of $2 \times 2$ with stride 1. The combination of the output of the convolutions and the non-time series inputs is then followed by two fully-connected layers with 50 and 20 neurons. The activation function used is the Rectified Linear Unit (ReLU) except for the output layer where no activation function is used. A sketch of the structure of the neural network is provided in Figure~\ref{fig:NN_architecture}.

 \def\layersep{2cm}
 \def\layersepb{1.5cm}
 
 \newcommand{\stack}[4]{
  \foreach \i in {1,...,#1} {
    \draw[convol] #2 ++({0.1*(#1)},{-0.1*(#1)}) ++({-0.1*\i},{0.1*\i}) rectangle +#3;
  }
}

\begin{figure}[ht!]
 \centering
  \resizebox{280px}{!}{
\begin{tikzpicture}[shorten >=1pt,->,draw=black!50, node distance=\layersep]
    \tikzstyle{every pin edge}=[<-,shorten <=1pt]
    \tikzstyle{neuron}=[circle,fill=black!25,minimum size=17pt,inner sep=0pt]
    \tikzstyle{sequence}=[fill=black!25,minimum width=30pt,minimum height=10pt,inner sep=0pt]
    \tikzstyle{convol}=[fill=blue!25,minimum width=25pt,minimum height=30pt,inner sep=0pt]
    \tikzstyle{convol_empty}=[minimum width=30pt,minimum height=30pt,inner sep=0pt]
    \tikzstyle{input neuron}=[neuron, fill=green!50];
    \tikzstyle{input sequence}=[sequence, fill=green!50];
    \tikzstyle{output neuron}=[neuron, fill=red!50];
    \tikzstyle{hidden neuron}=[neuron, fill=blue!25];
    \tikzstyle{annot} = [text width=10em, text centered]

    \foreach \name / \y in {1,...,2}
        \node[input sequence, pin=left:{\large Input \#\y}] (I-\name) at (0,-\y*1.5) {};

        \node[convol_empty] (C1) at (\layersep,-1*1.5) {};
	\stack{3}{(\layersepb,-1.5)}{(0.8,0.4)}{53}
        \node[convol_empty] (C1b) at (\layersep,-2*1.5) {};
	\stack{3}{(\layersepb,-3)}{(0.8,0.4)}{53}

        \node[convol_empty] (C-2) at (2*\layersep,-1.5*1.5) {};
	\stack{3}{(2.3*\layersepb,-0.4-2.25)}{(0.8,0.8)}{53}
	
    \foreach \name / \y in {3}
   	\node[input neuron, pin=left:{\large Input \#\y}] (I-\name) at (2*\layersep,-\y*1.35) {};
    \foreach \name / \y in {3}
   	\node[input neuron, pin=left:{ \vdots }] (I-\name) at (2*\layersep,-\y*1.35) {};

    \foreach \name / \y in {1,...,5}
        \path[yshift=-0.5cm]
            node[hidden neuron] (H-\name) at (3*\layersep,-\y cm) {};

    \foreach \name / \y in {1,...,3}
        \path[yshift=-1.5cm]
            node[hidden neuron] (H2-\name) at (4*\layersep,-\y cm) {};

    \node[output neuron, right of=H2-1, yshift=-0.3cm] (O1) {};
    \node[output neuron, right of=H2-3, yshift=+0.3cm] (O3) {};
    
    \path (I-1) edge (C1);
    \path (I-2) edge (C1b);
    \path (C1) edge (C-2);
    \path (C1b) edge (C-2);

    \foreach \source in {2,...,2}
        \foreach \dest in {1,...,5}
            \path (C-\source) edge (H-\dest);
    \foreach \source in {3,...,3}
        \foreach \dest in {1,...,5}
            \path (I-\source) edge (H-\dest);

    \foreach \source in {1,...,5}
        \foreach \dest in {1,...,3}
            \path (H-\source) edge (H2-\dest);

    \foreach \source in {1,...,3}
        \path (H2-\source) edge (O1);
    \foreach \source in {1,...,3}
        \path (H2-\source) edge (O3);

    \node[annot,above of=H-1, node distance=1.8cm, xshift=0.5*\layersep] (hl) {\large Fully-connected layers};
    \node[annot,above of=C1, node distance=1.5cm, xshift=0.5*\layersep] (convol) {\large Convolutions};
    \node[annot,above of=O1, node distance=2.7cm] (outputs) {\large Outputs};
\end{tikzpicture}
}

 \caption
 [Sketch of the structure of the neural network architecture.]
 {Sketch of the structure of the neural network architecture. 
The neural network processes the time series using a set of convolutional layers. The output of the convolutions and the other inputs are followed by fully-connected layers and the output layer.}
 \label{fig:NN_architecture}
\end{figure}
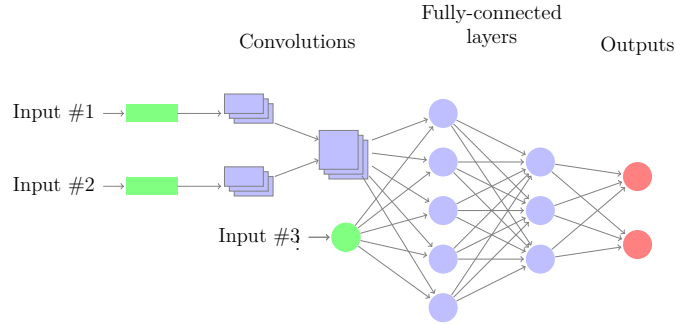

\subsubsection{Hyperparameters Used in DQN}
The different hyperparameters used for the DQN algorithm \cite{mnih2015human} are provided in Table \ref{tab:DQN_params_for_MG}.

\begin{table}[ht!]
\caption
[Main parameters of the DQN algorithm.]
{Main parameters of the DQN algorithm. The integer $k$ refers to the iteration number from the beginning of the training and the integer $l$ refers to the epoch number.}
\centering
\resizebox{0.65\textwidth}{!}{  
\begin{tabular}{| l | r |}
  \hline
  Parameter & \\
  \hline
    Update rule & RMSprop ($\rho_{RMS}=0.9$)\label{ntn:rho_RMS} \cite{rmsprop}\\
    Initial learning rate & $\alpha_0$=0.0002\\
    Learning rate update rule & $\alpha_{l+1}=0.99 \alpha_{l}$\\
    RMS decay & 0.9\\
    Initial discount factor & $\gamma_0$=0.9\\
    Discount factor update rule & $\gamma_{l+1}=\min(0.98,1-0.99 (1-\gamma_l))$\\
    Parameter of the $\epsilon$-greedy\label{ntn:epsilon-greedy} & $\epsilon=max(0.3,1-\frac{k}{500.000})$\\
    Replay memory size & 1000000\\
    Batch size & 32\\
    Freeze interval & 1000 iterations\\
    Initialization of the NN layers & Glorot uniform \cite{glorot2010understanding}\\
  \hline
\end{tabular}
}
\label{tab:DQN_params_for_MG}
\end{table}


\end{document}